\documentclass{article}

\PassOptionsToPackage{numbers, compress}{natbib}
\usepackage[main, final]{neurips_2026}

\usepackage{amsmath}
\usepackage{graphicx}
\usepackage{float}
\usepackage{amssymb}
\usepackage{algorithm}
\usepackage{algorithmic}
\usepackage{subcaption}
\usepackage{makecell}
\usepackage{multirow}
\usepackage{booktabs}
\usepackage{enumitem}
\usepackage{multicol}
\usepackage{amsthm}
\usepackage{thmtools}
\usepackage{thm-restate}
\usepackage{xcolor}

\usepackage[utf8]{inputenc} 
\usepackage[T1]{fontenc}    
\usepackage{hyperref}       
\usepackage{url}            
\usepackage{booktabs}       
\usepackage{amsfonts}       
\usepackage{nicefrac}       
\usepackage{microtype}      
\usepackage{xcolor}         

\title{Model-Agnostic FDR Control via Group Gaussian Mirror and Permutation SHAP}

\author{%
  \textbf{Jiaan Han}\textsuperscript{1,2},
  \textbf{Junxiao Chen}\textsuperscript{1,2},
  \textbf{Yanzhe Fu}\textsuperscript{1,3}\\
  \textsuperscript{1}University of Michigan\\
  \textsuperscript{2}Columbia University\\
  \textsuperscript{3}
  University of Hong Kong\\
  \texttt{\{jiaanhan, junxchen, yanzhefu\}@umich.edu}
}

\begin{document}

\maketitle

\begin{abstract}

Most FDR-controlled feature selection methods are designed for coordinate-wise hypotheses, where each feature has a single weight or importance score. This abstraction fails in sequential and grouped models, where one original feature is represented by a block of sub-features, such as lags, recurrent states, or attention-based interactions. We propose a grouped-feature FDR control framework for such settings. For grouped linear models, we construct null-symmetric block-level mirror statistics with matrix-valued perturbations. For neural sequential models, we combine Permutation SHAP derivatives as model-agnostic block-level importance scores with kernel-based dependence measure. The framework is model-agnostic across network architectures, does not require specifying the covariate distribution, and reduces to Gaussian Mirror or Neural Gaussian Mirror when the block size is one. We prove FDR control for low- and high-dimensional grouped linear models and asymptotic symmetry of smoothed Permutation SHAP derivatives under fixed fitted nonlinear models. Experiments on simulated and real-world datasets show reliable FDR control and improved power under correlated grouped-feature signals. Code is available at: \url{https://github.com/HJA321/Model\_Agnostic\_FDR\_Control}.

\end{abstract}

\section{Introduction}
\label{introduction}

Controlling the false discovery rate (FDR) is a central problem in feature selection, particularly when prediction and identifying truly relevant input variables are both important. In high-dimensional inference, selecting features without controlling FDR can lead to misleading scientific conclusions and unstable models. A large body of work has developed FDR-controlled procedures for feature selection, including the BH procedure \cite{Benjamini1995}, Knockoff-based methods \cite{Barber2014, Candes2016, Chi2021, Hansen2021, Lu2018, Zhu2021, Zuo2024, Shen2024}, Gaussian Mirror and Neural Gaussian Mirror \cite{Xing2019, Xing2020}, and Data Splitting \cite{Dai2020, Dai2020ASA, Sawaya2025}. Among these, Gaussian Mirror is particularly appealing for not requiring a specified joint distribution of the covariates, unlike Knockoffs, and has higher statistical power in general settings \cite{Ke2020}.

Despite their differences, most existing FDR-controlled feature selection methods are built around a common prototype: each feature is associated with a single model weight. In the classical linear model
\(
y = X\beta + \varepsilon,
\)
feature-level inference is typically formulated through coordinate-wise hypotheses
\[
H_{0,j}:\beta_j = 0, \qquad j = 1, \dots, p.
\]
This single-weight formulation underlies Gaussian Mirror, Knockoff-based procedures, and many related methods. Even in nonlinear settings \cite{Dai2020ASA, Hansen2021, Zhu2021, Zuo2024}, feature importance is often reduced to a scalar score per feature, preserving this one-feature-one-weight paradigm.

However, the one-feature-one-weight prototype is not the right abstraction for many grouped or sequential prediction problems. A clear example is a lagged time-series model
\(
Y_t
=
\sum_{\ell=0}^{m-1} X_{t-\ell}^{\top}\boldsymbol{\beta}_{\ell}
+
\varepsilon_t
\),
where each \(X_{t-\ell}\in\mathbb{R}^p\) contains the same \(p\) original features observed at lag \(\ell\). For an original feature \(j\), the model contains a lag-weight vector
\(
\boldsymbol{\beta}_j = (\beta_{j,0}, \beta_{j,1}, \ldots, \beta_{j,m-1}) .
\)
Thus, the scientific question should not be whether a single lag-specific weight \(\beta_{j,\ell}\) is zero, but whether feature \(j\) has any predictive contribution across the whole lag window, which formulates the new hypotheses:
\[
H_{0,j}: \beta_{j,0}=\beta_{j,1}=\cdots=\beta_{j,m-1}=0.
\]
Equivalently, if \(X_j=[\boldsymbol{x}_{j,t}, \boldsymbol{x}_{j,t-1}, \ldots, \boldsymbol{x}_{j,t-m+1}]\in\mathbb{R}^{n\times m}\), the feature-level null is a grouped null. Although this example is written as a lagged linear model, the same question appears whenever an original feature is represented by a structured block, including grouped regression and sequential networks, where it may affect output through multiple time steps, hidden states, or attention channels. The inferential target remains in feature-level, but the testing object is now its grouped input block.

Motivated by this grouped feature-level null, we propose a model-agnostic FDR control framework for grouped and sequential models. For linear lagged or grouped models, we develop a \emph{Group Gaussian Mirror} procedure that replaces scalar mirror variables with block-level mirror constructions and yields null-symmetric group-level mirror statistics. For sequential neural networks, our method combines matrix-valued mirror perturbations with Derivatives of Permutation SHAP \cite{Lundberg2017, Mitchell2021}, which provide model-agnostic block-level importance scores by differentiating each feature's SHAP value with respect to its sub-features. This makes the framework applicable to recurrent and attention-based architectures such as LSTM \cite{Hochreiter1997}, GRU \cite{Chung2014}, and Transformer encoder \cite{Vaswani2017} models.

We provide theoretical guarantees for both the linear and neural network versions of the framework. We prove null symmetry and FDR control for low- and high-dimensional grouped linear models, using a two-stage screen-then-mirror strategy in high dimensions. For neural models, we justify smoothed Permutation SHAP derivatives as asymptotically symmetric feature-importance scores under a fixed fitted nonlinear model and establish a symmetry-safe principle in high dimensions. Empirically, we evaluate the proposed framework across linear lagged models and sequential neural architectures, including LSTM, GRU, and Transformer encoder models, in both low- and high-dimensional settings. Our simulations cover various scenarios involving covariates correlation, lag-block dependence, nonlinear relationships and signal strength. Across simulated and real-world datasets, our method maintains reliable FDR control and improves power, especially when signals are distributed across correlated feature blocks. We provide detailed algorithms in Appendix \ref{algorithms}, flowchart in Appendix \ref{flowchart} and proofs of all theorems and propositions in Appendix \ref{proofs}.

Our main contributions are threefold: we extend FDR control from coordinate-wise hypotheses to grouped-feature hypotheses; we provide a model-agnostic FDR framework for neural sequential models; and we introduce Permutation-SHAP derivatives as block-level feature importance scores for constructing the mirror-statistic in FDR control.

\textbf{Notations:} Uppercase letters denote matrices, bold lowercase letters denote vectors, and lowercase letters denote scalars. We write \(X=[X_1,\ldots,X_p]\), where \(X_j\in\mathbb R^{n\times m}\) is the block for feature \(j\), and \(X_{-j}\) is the design with this block removed. Let \(P_A\) and \(M_A=I_n-P_A\) denote projection and residual-maker matrices. Let \(\mathcal S_1\) and \(\mathcal S_0\) denote the sets of relevant and null features, respectively.

\section{Background}
\label{background}

\subsection{Gaussian Mirror}
\label{gm_background}

Gaussian Mirror \cite{Xing2019} controls FDR by constructing a pair of mirror variables
\(
x_j^\pm=x_j \pm c_jz_j
\)
for each feature \(x_j\), where \(z_j \sim N(0,I_n)\) is independent Gaussian noise. The perturbation \(c_j\) is chosen so that \(x_j^+\) and \(x_j^-\) are conditionally uncorrelated given \(X_{-j}\). In a linear model, this makes the corresponding fitted weights \(\widehat\beta_j^+\ ,\widehat\beta_j^-\) independent centered Gaussian variables under the null \(H_{0,j}:\beta_j=0\), while relevant features tend to produce two mirror weights with aligned nonzero signal. 

The key property for FDR control is the sign symmetry of the mirror statistic under null features. A best choice under most conditions is the "signed-sum" \cite{Dai2020, Dai2020ASA}
\[
M_j
=
\operatorname{sign}\left(\widehat\beta_j^+,\widehat\beta_j^-\right)
\left(|\widehat\beta_j^+|+|\widehat\beta_j^-|\right).
\]
For null features, \(M_j\) is symmetric around zero, whereas for relevant features it is likely to be large and positive. Hence the negative statistics provide an estimate of the number of false discoveries among the positive statistics. For target FDR level \(q\), it selects features with \(M_j\ge \tau_q\), where \cite{Candes2016}
\[
\tau_q
=
\min\left\{
t>0:
\frac{\#\{j:M_j\le -t\}+1}{\#\{j:M_j\ge t\}\vee 1}
\le q
\right\}.
\]
This mirror statistic principle is the basis for our block-level extension.

\subsection{Neural Gaussian Mirror}
\label{ngm_background}

Neural Gaussian Mirror \cite{Xing2020} extends the mirror construction to nonlinear models. In the linear case, the key construction target is conditional linear uncorrelation between the two mirror variables. For neural networks, this is replaced by the stronger conditional-independence target
\[
U\perp V\mid W,
\quad \text{where} \quad
U=x_j+c_jz_j,\quad
V=x_j-c_jz_j,\quad
W=X_{-j}.
\]
To obtain a computable objective, it places the log-density \(\eta(u,v,w)=\log p_{U,V,W}(u,v,w)\) in a tensor-product Reproducing Kernel Hilbert Space (RKHS). Its interaction can be decomposed as
\[
\eta
=
\eta_U+\eta_V+\eta_W+\eta_{U,W}+\eta_{V,W}
+\eta_{U,V}+\eta_{U,V,W}.
\]
The conditional-independence null is characterized by the absence of interaction components involving both mirror branches:
\[
U\perp V\mid W
\quad\Longleftrightarrow\quad
\eta_{U,V}+\eta_{U,V,W}=0.
\]
More formally, each marginal RKHS is decomposed as
\(\mathcal H_A=\mathcal H_A^0\oplus \mathcal H_A^1\) for
\(A\in\{U,V,W\}\), where \(\mathcal H_A^0\) and
\(\mathcal H_A^1\) denote the constant and centered components. Thus the full tensor-product space admits the decomposition
\[
\mathcal H
= \bigoplus_{a,b,c\in\{0,1\}}
\mathcal H_{abc}
= \bigoplus_{a,b,c\in\{0,1\}}
\mathcal H_U^a\otimes \mathcal H_V^b\otimes \mathcal H_W^c .
\]
The conditional-independence null subspace and the alternative interaction subspace are
\[
\mathcal H_0
=
\mathcal H_{000}\oplus \mathcal H_{100}\oplus \mathcal H_{010}
\oplus \mathcal H_{001}\oplus \mathcal H_{101}\oplus \mathcal H_{011},
\quad
\mathcal H_1
=
(\mathcal H_U^1\otimes \mathcal H_V^1\otimes \mathcal H_W^0)
\oplus
(\mathcal H_U^1\otimes \mathcal H_V^1\otimes \mathcal H_W^1).
\]
The likelihood-ratio score in this alternative interaction space yields the kernel objective
\[
I_j^K(c)^2
=
\frac{1}{n^2}
\left[
(H_nK_UH_n)\circ(H_nK_VH_n)\circ K_W
\right]_{++},
\]
where \(H_n=I_n-\boldsymbol 1\boldsymbol 1^\top/n\), \(\circ\) is the Hadamard product, and \([]_{++}\) is the sum of all entries. Neural Gaussian Mirror chooses \(\widehat c_j \in \arg\min_{c\ge0}I_j^K(c)^2\). Our method for neural networks builds upon this RKHS conditional-dependence objective, while extending it to block-valued mirror variables.

\section{Method for Linear Models}
\label{linear_ggm}

\subsection{Block-Level Mirror Statistics}
\label{blockmirrorstat}

The original Gaussian Mirror is designed for coordinate-wise hypotheses, where one feature corresponds to one weight. In grouped models, the null object should be a whole weight block.

\newtheorem{proposition}{Proposition}
\begin{restatable}{proposition}{gmfailgrouped}
\label{prop:gm_fail_grouped}
In a grouped design \(X=[X_1,\ldots,X_p]\) with
\(X_j\in\mathbb R^{n\times m}\) and
\(\boldsymbol\beta_j\in\mathbb R^m\), applying the Gaussian Mirror separately to each sub-feature of \(X_j\)
targets the coordinate-wise nulls \(H_{0,(j,\ell)}:\beta_{j,\ell}=0\) for \(\ell=0,\ldots,m-1\), rather than the grouped null \(H_{0,j}:\boldsymbol\beta_j=(\beta_{j,0}, \dots, \beta_{j,m-1}) = \boldsymbol 0\).
\end{restatable}

Proposition~\ref{prop:gm_fail_grouped} shows that once the inferential target becomes grouped, the mirror construction must also be lifted from the scalar level to the block level. For each feature \(j\), we therefore introduce a pair of block-level mirror variables of the form
\[
X_j^+ = X_j + \widetilde Z_j G_j,
\qquad
X_j^- = X_j - \widetilde Z_j G_j.
\]
where \(\widetilde Z_j \in \mathbb{R}^{n \times m}\) is a transformed noise matrix and \(G_j \in \mathbb{R}^{m \times m}\) is a group perturbation matrix. In the original Gaussian Mirror, the perturbation \(c_j \boldsymbol{z}_j\) is chosen so that the two residualized mirror variables are orthogonal. Here, the matrix perturbation \(\widetilde Z_j G_j\) plays the same role at the group level, but now the relevant object is an \(m \times m\) block cross-product rather than a scalar covariance term.

Let \(\widehat{\boldsymbol\beta}_j^+\) and \(\widehat{\boldsymbol\beta}_j^-\) denote the fitted weight vectors of \(X_j^+\) and \(X_j^-\), respectively, in the regression on \([X_{-j},X_j^+,X_j^-]\). We then form a block-level mirror statistic from this pair to replace the scalar one, which preserves the properties of the previous "signed sum" in Section \ref{gm_background}. A concrete example used throughout this paper is:
\[
M_j
=
\operatorname{sign}\!\left(\langle \widehat{\boldsymbol\beta}_j^+,\widehat{\boldsymbol\beta}_j^- \rangle\right)
\cdot
\left(
\|\widehat{\boldsymbol\beta}_j^+\|_2 + 
\|\widehat{\boldsymbol\beta}_j^-\|_2
\right).
\]

\subsection{Low-Dimensional Group Gaussian Mirror}
\label{subsec:low_dim_group_gm}

We now give the explicit Group Gaussian Mirror construction in low-dimensional settings. Fix a grouped feature \(j\), let
\(
R_j = M_{X_{-j}} X_j
\)
be the residualized group block, and \(Z_j \in \mathbb{R}^{n \times m}\) be an i.i.d. Gaussian noise matrix independent of \((X,\boldsymbol y)\). We define
\[
Q_j = \bigr(I_n - P_{[X_{-j},R_j]}\bigl) Z_j,
\qquad
\widetilde Z_j = Q_j \bigr(Q_j^\top Q_j\bigl)^{-1/2},
\qquad
G_j = \bigr(R_j^\top R_j\bigl)^{1/2}.
\]
The low-dimensional Group Gaussian Mirror variables are then constructed as \(X_j^\pm = X_j \pm \widetilde Z_j G_j\). 

Theorem \ref{thm:low_dim_ggm} shows the grouped analogue exactly inherits the key properties of the original Gaussian Mirror. Given the block-level mirror statistics \(\{M_j\}_{j=1}^p\), Group Gaussian Mirror uses the same mirror thresholding and selection rule as in \ref{gm_background}, and effectively controls the FDR at the targeted level as shown in Proposition \ref{prop:full}.

\newtheorem{theorem}{Theorem}
\begin{restatable}{theorem}{lowdimggm}
\label{thm:low_dim_ggm}
Assume that \(R_j^\top R_j\) and
\(Q_j^\top Q_j\) are non-singular. For a grouped linear model with \(X_j^\pm=X_j\pm \widetilde Z_jG_j\), the coefficient vectors \(\widehat{\boldsymbol\beta}_j^+\) and 
\(\widehat{\boldsymbol\beta}_j^-\) satisfy \(\mathbb E[\widehat{\boldsymbol\beta}_j^+]
= \mathbb E[\widehat{\boldsymbol\beta}_j^-]
= \boldsymbol 0\), \(\operatorname{Cov}(\widehat{\boldsymbol\beta}_j^+)
= \operatorname{Cov}(\widehat{\boldsymbol\beta}_j^-)\), \(\operatorname{Cov}(\widehat{\boldsymbol\beta}_j^+,\widehat{\boldsymbol\beta}_j^-)
= 0\), and are jointly Gaussian under \(H_{0,j}\), which yields null symmetry of the block-level mirror statistic \(M_j\). Alternatively, \(\mathbb E[\widehat{\boldsymbol\beta}_j^+]=\mathbb E[\widehat{\boldsymbol\beta}_j^-]=\frac{1}{2}\boldsymbol\beta_j\) under \(H_{1,j}\).
\end{restatable}

\begin{restatable}{proposition}{fullmethod}
\label{prop:full}
Given the properties in Theorem \ref{thm:low_dim_ggm}, Group Gaussian Mirror asymptotically controls the FDR at the target level \(q\).
\end{restatable}

The following Proposition \ref{prop:gls_whitened_ggm} also shows that the same construction can be extended to GLS estimators when we assume \(\varepsilon\) are not i.i.d.. In Experiment \ref{linear_simulation}, however, we find the OLS version robust enough under correlated \( \varepsilon \).

\begin{restatable}{proposition}{glswhitenedggm}
\label{prop:gls_whitened_ggm}
For a grouped linear model with gaussian noise \(\varepsilon\sim N(0,\sigma^2\Sigma)\), where \(\Sigma\succ 0\), let \(L\) satisfy \(L^\top L=\Sigma^{-1}\), define \(\bar Y=LY\) and \(\bar X_k=LX_k\) for \(k=1,\ldots,p\). Then applying the OLS Group Gaussian Mirror to the whitened system \((\bar Y, \bar X, \bar\varepsilon)\) is equivalent to applying GLS Group Gaussian Mirror to the original system \((Y, X, \varepsilon)\) with \(\Sigma^{-1}\)-inner product.
\end{restatable}

\subsection{High-Dimensional Group Gaussian Mirror}
\label{highdimggm}

When \(pm>n\), the OLS-based Group Gaussian Mirror is no longer directly applicable. Moreover, the coordinate-wise post-selection argument used in the original high-dimensional Gaussian Mirror \cite{Xing2019} relies on the polyhedral representation of the lasso selection event in the response vector. The following Proposition \ref{prop:post_selection_not_grouped} shows that this argument is not invariant under the replacement from \(\ell_1\)-Lasso to Group Lasso \cite{Yuan2006} and thus cannot be directly applied either.

\begin{restatable}{proposition}{postselectionnotgrouped}
\label{prop:post_selection_not_grouped}
For a grouped linear model, assume the weight obtained from the first-stage Group Lasso screener is \(\widehat{\boldsymbol\beta}^{\,GL}\), and let \(\widehat{\mathcal A}=\{j:\widehat{\boldsymbol\beta}^{\,GL}_j\neq \boldsymbol 0\}\) be the selected grouped feature set. Then, in general, the event \(\widehat{\mathcal A}=A\) is not polyhedral in
\(Y\).
\end{restatable}

Therefore, we adopt a two-stage procedure in light of \cite{Dai2020}. We split the sample into two independent (or asymptotically independent) chunks, perform first-stage screening on the first chunk, then use our algorithm in the second chunk. For time-series models, we divide the observations into two consecutive chunks and leave a short buffer gap between them; under stationarity and a suitable mixing condition, the cross-chunk dependence becomes asymptotically negligible \cite{Yu1994, Bradley2005}. Let \((X^{(1)},\boldsymbol y^{(1)})\) and \((X^{(2)},\boldsymbol y^{(2)})\) denote the two resulting data chunks. In the first stage, we apply a grouped feature screening method to the first chunk, and obtain a selected feature set \(\widehat{\mathcal A} \subseteq \{1,\ldots,p\}\). In the second stage, we restrict the design to the selected grouped features \(\{X_j^{(2)}:j\in\widehat{\mathcal A}\}\), and run exactly the same low-dimensional Group Gaussian Mirror procedure as in Section \ref{subsec:low_dim_group_gm} on the second chunk only. For grouped features not selected in the first stage, their mirror statistics are set to zero.

\begin{restatable}{theorem}{twostageggm}
\label{thm:two_stage_ggm}
Assume that: (i) the two data chunks share the same grouped-feature null structure and are independent, or asymptotically independent; (ii) the first-stage screening procedure satisfies a sure screening property, namely \(\mathbb P(\mathcal S_1 \subseteq \widehat{\mathcal A}) \to 1\); and (iii) with probability tending to one, the second-stage restricted design is low-dimensional and full-rank. Then for every null grouped feature \(j\in\mathcal S_0\), the two-stage mirror statistic \(\widetilde M_j\) satisfies
\[
\sup_{t>0}
\left|
\mathbb P(\widetilde M_j>t)-\mathbb P(\widetilde M_j<-t)
\right|
\to 0.
\]
\end{restatable}

Theorem~\ref{thm:two_stage_ggm} formalizes the key point that the first stage is used only for dimension reduction, while the null symmetry is inherited entirely from the second stage. Consequently, applying the standard mirror thresholding rule to \(\{\widetilde M_j\}_{j=1}^p\) yields asymptotic FDR control. Under suitable conditions, Group Lasso is known to have consistent group-level sure screening \cite{Wei2010} and therefore be our default choice. More generally, any method with the sure screening property can be used in the first stage.

\section{Method for Neural Networks}
\label{neuralnetwork}

\subsection{Permutation SHAP Derivative as Feature Importance}
\label{pshap_derivative}

For FDR control in neural networks, we first need a feature-importance metric that is stable and compatible with grouped-feature inference. Existing neural knockoff and Neural Gaussian Mirror methods typically use variants of fitted derivatives \cite{Hechtlinger2016}, obtained by perturbing one input coordinate while holding all others fixed. This is not well suited to our grouped or sequential setting. Perturbing an entire grouped feature jointly no longer yields a standard partial derivative, and differentiating each sub-feature separately does not directly represent the overall effect of the grouped feature and can be highly unstable when sub-features within a block are strongly correlated. 

Therefore, we propose Permutation SHAP \cite{Lundberg2017, Mitchell2021} derivatives as our model-agnostic feature importance metric. For a fixed fitted nonlinear model \(f\), let \(\Pi\) be a uniformly random permutation
of the grouped feature indices, \(P_j(\pi)\) denote the set of features appearing before feature \(j\) in a given permutation order \(\pi\), \(X'\) be an independent background draw, and \((\boldsymbol{x}_S,\boldsymbol{X}'_{-S})\) denote the complete hybrid input vector whose feature blocks in \(S\) are taken from \(\boldsymbol{x}\) and whose remaining blocks are taken from \(X'\). The single-permutation marginal contribution of feature \(j\) is defined as
\[
\Delta_j(\boldsymbol{x};\pi,X')
:=
f\bigl(\boldsymbol{x}_{P_j(\pi)\cup\{j\}},X'_{-(P_j(\pi)\cup\{j\})}\bigr)
-
f\bigl(\boldsymbol{x}_{P_j(\pi)},X'_{-P_j(\pi)}\bigr),
\]

We then define the exact Permutation SHAP value \(\Phi_j(\boldsymbol{x})\) and its derivative \(\psi_j(\boldsymbol{x})\) by
\[
\Phi_j(\boldsymbol{x})
:=
\mathbb{E}_{\Pi,X'}\!\left[\Delta_j(\boldsymbol{x};\Pi,X')\right],
\quad \text{and} \quad
\boldsymbol{\psi}_j(\boldsymbol{x})
:=
\nabla_{\boldsymbol{x}_j}\Phi_j(\boldsymbol{x}).
\]
In our grouped-feature setting, specifically, for grouped feature \(j\) with sub-features \(x_{j,1},\ldots,x_{j,m}\), we compute Permutation SHAP for the grouped feature and differentiate it with respect to each sub-feature, yielding a block-level importance vector analogous to \(\boldsymbol{\beta}_j\) in linear models:
\[
\boldsymbol{\psi}_j(\boldsymbol{x})
=
\left(
\frac{\partial \Phi_j(\boldsymbol{x})}{\partial x_{j,1}},
\ldots,
\frac{\partial \Phi_j(\boldsymbol{x})}{\partial x_{j,m}}
\right)^\top .
\]
For each sub-feature \(x_{j,i}\), this method produces \(n\) derivative values across the \(n\) explained points. We aggregate them by taking the average: \(\widehat{\boldsymbol\psi}_j^\pm=\frac{1}{n}\sum_{t=1}^n\widehat{\boldsymbol\psi}_j^\pm(\boldsymbol x_t)\). This aggregation targets the average SHAP-derivative effect; hence the neural grouped-feature null is interpreted as zero average block-level effect after aggregation, without requiring the local derivative effects to be time-invariant. The neural mirror statistic \(M_j\) is defined similarly to the construction in Section \ref{blockmirrorstat}:
\[
M_j=\operatorname{sign}\left(\langle \widehat{\boldsymbol\psi}_j^+,\widehat{\boldsymbol\psi}_j^- \rangle\right)\left(\|\widehat{\boldsymbol\psi}_j^+\|_2+\|\widehat{\boldsymbol\psi}_j^-\|_2\right).
\]

\begin{restatable}{proposition}{linearshapderivative}
\label{prop:linear_shap_derivative}
For a grouped linear model, let \(\Phi_j(\boldsymbol x)\) be the group-level Permutation SHAP value of group \(j\). Then \(\Phi_j(\boldsymbol x)=\boldsymbol\beta_j^\top\left(\boldsymbol x_j-\mathbb E[X'_j]\right)\) and hence \(\boldsymbol{\psi}_j(\boldsymbol{x})=\nabla_{\boldsymbol x_j}\Phi_j(\boldsymbol x)=\boldsymbol\beta_j\), which coincides with the linear-case coefficient vector \(\boldsymbol\beta_j\).
\end{restatable}

These Permutation SHAP derivatives can be viewed as a coalition-weighted fitted derivatives, and Proposition \ref{prop:linear_shap_derivative} shows that it is a direct nonlinear analogue of the coefficient vector in a grouped linear model. In practice, we approximate \(\Phi_j(\boldsymbol{x})\) by Monte Carlo Permutation SHAP and estimate its derivative by local-quadratic smoothing with a tricube kernel and \(\lceil0.3n\rceil\)-nearest-neighbor bandwidth. 

\begin{restatable}{theorem}{shapderivativesymmetry}
\label{thm:shap_derivative_symmetry}
For a fixed nonlinear model \(f\) and a null grouped feature \(j\), let \(\widehat\psi_{j,k}(x_0)\) denote the local-quadratic smoothed Monte Carlo Permutation SHAP derivative estimator. Under some regularity conditions, we have, \(\widehat\psi_{j,k}(x_0)\) is symmetric around zero for each sub-feature
\(x_{j,k}\), namely, 
\[
\sup_{t>0}
\left|
\mathbb P\!\left(\widehat\psi_{j,k}(x_0)>t\right)
-
\mathbb P\!\left(\widehat\psi_{j,k}(x_0)<-t\right)
\right|
\to0, \quad \forall k.
\]
\end{restatable}

Theorem~\ref{thm:shap_derivative_symmetry} justifies using the aggregated block-level SHAP derivative vector as the neural mirror importance score. In the experiment, we find that the number of permutation draws does not affect the result heavily. As detailed in Appendix \ref{mcexperiment}, we find that 10 draws are sufficient and performs stably, and we fix this number in our later experiments.

\subsection{Low-Dimensional Permutation SHAP Gaussian Mirror}
\label{low_dim_group_ngm}

We now introduce our method in neural networks, which combines the block-level Permutation SHAP feature importance in Section \ref{pshap_derivative} with a grouped Neural Gaussian Mirror construction. The mirror variables follow the linear Group Gaussian Mirror. For feature \(j\), we define 
\[
U_j(G_j)=X_j+\widetilde Z_jG_j, \quad
V_j(G_j)=X_j-\widetilde Z_jG_j, \quad
W_j=X_{-j}
\]
where \(\widetilde Z_j\in\mathbb R^{n\times m}\) is the Gaussian noise block and \(G_j\in\mathbb R^{m\times m}\) is the perturbation matrix. The difference is that \(G_j\) can no longer be chosen by the closed-form block orthogonality condition in Section \ref{subsec:low_dim_group_gm}; instead, it is optimized through the kernel-based conditional-dependence objective introduced in Section \ref{ngm_background}. The target remains \(U_j(G_j)\perp V_j(G_j)\mid W_j\), while \(U_j(G_j)\) and \(V_j(G_j)\) are now being block-valued mirror branches.

\begin{restatable}{theorem}{groupngmobjective}
\label{thm:group_ngm_objective}
Applying RKHS log-density decomposition and likelihood-ratio score construction to \((U_j(G_j),V_j(G_j),W_j)\) gives the grouped Neural Gaussian Mirror objective:
\[
I_j^K(G_j)^2
=
\frac{1}{n^2}
\left[
(H_nK_U(G_j)H_n)\circ(H_nK_V(G_j)H_n)\circ K_W
\right]_{++},
\]
where \(K_U\), \(K_V\) and \(K_W\) are kernel matrices computed from the rows of \(U_j(G_j)\), \(V_j(G_j)\) and \(W_j\).
\end{restatable}

Theorem~\ref{thm:group_ngm_objective} shows that the grouped construction changes the perturbation \(G_j\) from a scalar to a matrix, while the conditional dependence criterion remains a scalar RKHS score norm. In all experiments, \(K_U,K_V,K_W\) are Gaussian RBF kernels with median pairwise-distance bandwidths after column standardization. We therefore choose
\(
\widehat G_j\in\arg\min_{G_j\in\mathcal G} I_j^K(G_j)^2 .
\)

\begin{restatable}{proposition}{linearkernelreduction}
\label{prop:linear_kernel_reduction}
Under linear kernels and ignoring the \(W_j\)-kernel weight, the objective satisfies
\[
I_j^K(G_j)^2
\propto
\left\|
\frac{1}{n}
\bigl(H_nU_j(G_j)\bigr)^\top
\bigl(H_nV_j(G_j)\bigr)
\right\|_F^2,
\]
which reduces to the block-orthogonality target in the linear Group Gaussian Mirror.
\end{restatable}

Proposition~\ref{prop:linear_kernel_reduction} shows that the grouped neural construction reduces to the block-orthogonality principle of linear GGM under linear kernels. Empirically, the objective \(I_j^K(G_j)^2\) is well behaved: full-matrix optimization initialized at \(I_m\) converges stably even with a relatively large learning rate. Appendix~\ref{convergenceexperiment} provides the convergence results, and we fix learning rate \(0.5\) and 200 epochs in later experiments.

\subsection{High-Dimensional Permutation SHAP Gaussian Mirror}
\label{high_dim_group_ngm}

For high-dimensional neural networks, the data-splitting and Group Lasso screening in Section~\ref{highdimggm} are not directly applicable, since they are designed for OLS or GLS inference; exact OLS residualization from Section~\ref{subsec:low_dim_group_gm} is also unavailable. We therefore use a label-free ridge-type projection to construct a stable perturbation basis \(\widetilde Z_{j,\lambda}\). This projection is not meant to remove nonlinear dependence, which is handled by \(I_j^K(G_j)^2\), but to place mirror noise in a stable incremental direction while preserving the sign-flip symmetry required by the mirror principle. We define
\[
Q_{j,\lambda}=A_{j,\lambda}(X_j,W_j)Z_j,
\quad
\widetilde Z_{j,\lambda}
=
Q_{j,\lambda}
\left(
Q_{j,\lambda}^\top Q_{j,\lambda}+\varepsilon I_m
\right)^{-1/2},
\quad \text{with} \quad
\varepsilon>0.
\]
Here, \(A_{j,\lambda}(X_j,W_j)\in\mathbb R^{n\times n}\) is a ridge-type projection operator, and \(Z_j\in\mathbb R^{n\times m}\) is an auxiliary Gaussian noise matrix independent of \((X,\boldsymbol y)\) and \(A_{j,\lambda}(X_j,W_j)\). In later high-dimensional experiments, \(A_{j,\lambda}\) is implemented by ridge residualization using ridge hat matrices, with \(\lambda_W=\lambda_Z=0.1n\).

\begin{restatable}{theorem}{ridgesymmetrysafe}
\label{thm:ridge_symmetry_safe}
For any grouped feature \(j\), let \(Z_j\in\mathbb R^{n\times m}\) independently generated and satisfy \(Z_j\overset d=-Z_j\) and \(Z_j\perp (X,\boldsymbol y)\). Assume \(\widetilde{Z}_{j,\lambda}\) is odd and \(\widehat{G}_j\) is even in \(Z_j\), and the downstream training-and-attribution procedure is swap-equivariant in distribution under exchanging the two mirror branches. Then the mirror outputs \(O_j^\pm\) are exchangeable in distribution, namely,
\(
(O_j^+,O_j^-)\overset{d}{=}(O_j^-,O_j^+).
\)

\end{restatable}

Theorem~\ref{thm:ridge_symmetry_safe} shows that the validity does not rely on exact OLS orthogonality. This theorem gives a structural requirement that the construction pipeline should be symmetry-safe. The ridge-type operator \(A_{j,\lambda}\) satisfies the requirement that the projection is odd in \(Z_j\), as it is label-free and linear in \(Z_j\). Thus, the ridge-type projection preserves all the geometry required by the mirror principle. Combined with the Permutation SHAP derivative metric from Section \ref{pshap_derivative}, which is swap-equivariant, our high-dimensional Neural Gaussian Mirror construction preserves the required structural exchangeability. Together, Theorems ~\ref{thm:shap_derivative_symmetry} and ~\ref{thm:ridge_symmetry_safe} establish key theoretical properties underlying the validity of PSGM under their respective assumptions. However, extending these guarantees to a full FDR control guarantee for the trained neural network requires additional conditions on systematic null-feature attribution. This challenge is closely related to predictive and explanation multiplicity and the Rashomon set, as further discussed in Appendix ~\ref{multiplicity}.

\section{Numerical Simulation}
\label{simulation}

\subsection{Linear Model Simulations}
\label{linear_simulation}

We evaluate our Group Gaussian Mirror (GGM) method on grouped linear designs, using both the OLS and the GLS version and compare it with simple Group Lasso. We also compare it with a Partial-\(F\) BH procedure, where we compute group-level \(p\)-values by partial \(F\)-tests for \(H_{0,j}:\boldsymbol{\beta}_j=\mathbf{0}\), and apply the standard Benjamini--Hochberg procedure to the resulting \(p\)-values \cite{Benjamini1995}. In all settings, data are generated from 
\(\boldsymbol{y}=X\boldsymbol{\beta}+\boldsymbol{\varepsilon}\), 
where the noise sequence \(\{\varepsilon_t\}_{t=1}^n\) follows a stationary AR(1) process \(\varepsilon_t=\phi\varepsilon_{t-1}+u_t\), \(u_t \overset{\text{iid}}{\sim} N(0,1-\phi^2)\), \(\varepsilon_0\sim N(0,1)\) with \(\phi=0.5\). We set grouped feature size \(m=5\), target FDR level \(q=0.1\), and \( | \mathcal S_1 | =30\) relevant grouped features. All weights in null grouped features are set to zero. For relevant grouped features, the signal scale is \(\sigma_\beta=20\sqrt{\log(| \mathcal S_1 |)/n}\). We use an AR grouped design: for each non-null grouped feature \(j\), we draw \(a_j\sim N(0,\sigma_\beta^2)\), then set
\(\beta_{j,\ell}=w_\ell a_j\), where \(w_\ell\) decreases linearly from 1.0 to 0.4 across lags. This creates signals that are spread across the whole lag block but weaker at more distant lags. We consider two dimension settings: a low-dimensional setting with \(n=2500\), \(p=300\), \(m=5\), and a high-dimensional setting with \(n=2500\), \(p=700\), \(m=5\). In the high-dimensional setting, we use the screen-then-mirror version of Group Gaussian Mirror, splitting the observations into two consecutive blocks with a gap of 10 observations between them. All experiments were carried out on a Supermicro AS-4124GO-NART+ server with eight NVIDIA A100-SXM4-80GB GPUs.

Each grouped feature is generated from \(m\) lagged values of a latent stationary AR(1) process. We vary the feature-level AR coefficient \(\rho_w\in\{0.1,0.5,0.9\}\), which controls the temporal dependence among lagged sub-features within the same grouped feature. Cross-feature dependence is introduced through the contemporaneous correlation \(\rho_b=\alpha\rho_w\) between latent processes, with \(\alpha\in\{0.2,0.8\}\).

For each setting, we repeat the experiment 50 times and report grouped-feature FDR and power averaged over repetitions. FDR is computed as the fraction of selected grouped features that are null, and power is computed as the fraction of non-null grouped features selected. The results are summarized in Table~\ref{tab:linear_simulation1}. In all tables below, \textbf{bold} numbers represent the best result and \textcolor{blue}{blue} numbers represent the second best. Empirically, we find the OLS version is sufficient even for correlated noise terms. We also conduct additional linear model experiments in Appendix \ref{additionalsimulation}, including experiments under a block-correlated Gaussian design and comparison with the Group Knockoff baseline \cite{Dai2016}.

\begin{table}[H]
\centering
\footnotesize
\setlength{\tabcolsep}{5pt}
\renewcommand{\arraystretch}{0.9}
\caption{Linear simulation results under the AR grouped design}
\label{tab:linear_simulation1}
\begin{tabular}{@{}ccccccccccc@{}}
\midrule
\multirow{2}{*}{\textbf{\makecell{Setting \\ (p*m, n)}}} &
\multirow{2}{*}{\textbf{\makecell{Feature \\ Corr}}} &
\multirow{2}{*}{\textbf{\makecell{AR \\ Coef}}} &
\multicolumn{2}{c}{\textbf{GGM-OLS}} &
\multicolumn{2}{c}{\textbf{GGM-GLS}} &
\multicolumn{2}{c}{\textbf{Partial-F BH}} &
\multicolumn{2}{c}{\textbf{Group Lasso}} \\
\cmidrule(lr){4-5}
\cmidrule(lr){6-7}
\cmidrule(lr){8-9}
\cmidrule(lr){10-11}
& & & FDR & Power & FDR & Power & FDR & Power & FDR & Power \\
\midrule

\multirow{6}{*}{$\makecell{p*m = 1500 \\ n = 2500}$}
& \multirow{3}{*}{0.2}
& 0.1 & 0.135 & 1.000 & \textcolor{blue}{0.119} & 1.000 & \textbf{0.102} & 1.000 & 0.412 & 1.000 \\
&
& 0.5 & \textbf{0.020} & 1.000 & 0.069 & 1.000 & \textcolor{blue}{0.053} & 1.000 & 0.402 & 1.000 \\
&
& 0.9 & \textbf{0.026} & 1.000 & 0.090 & 1.000 & \textcolor{blue}{0.069} & 1.000 & 0.457 & 0.999 \\
&
\multirow{3}{*}{0.8}
& 0.1 & \textcolor{blue}{0.077} & 1.000 & 0.078 & 1.000 & \textbf{0.062} & 1.000 & 0.367 & 1.000 \\
&
& 0.5 & \textbf{0.043} & 0.997 & 0.129 & 0.997 & \textcolor{blue}{0.065} & 0.997 & 0.402 & 0.999 \\
&
& 0.9 & \textbf{0.024} & 0.995 & 0.090 & 0.995 & \textcolor{blue}{0.081} & 0.995 & 0.449 & 0.996 \\

\midrule

\multirow{6}{*}{$\makecell{p*m = 3500 \\ n = 2500}$}
& \multirow{3}{*}{0.2}
& 0.1 & \textcolor{blue}{0.104} & 1.000 & 0.114 & 1.000 & \textbf{0.080} & 1.000 & 0.402 & 1.000 \\
&
& 0.5 & \textbf{0.033} & 1.000 & 0.067 & 1.000 & \textcolor{blue}{0.063} & 1.000 & 0.408 & 1.000 \\
&
& 0.9 & \textbf{0.024} & 1.000 & 0.081 & 1.000 & \textcolor{blue}{0.076} & 0.998 & 0.463 & 1.000 \\
&
\multirow{3}{*}{0.8}
& 0.1 & \textcolor{blue}{0.087} & 1.000 & 0.092 & 0.996 & \textbf{0.058} & 1.000 & 0.389 & 0.997 \\
&
& 0.5 & \textbf{0.046} & 1.000 & 0.123 & 1.000 & \textcolor{blue}{0.062} & 0.997 & 0.427 & 1.000 \\
&
& 0.9 & \textbf{0.019} & 0.997 & 0.099 & 0.998 & \textcolor{blue}{0.085} & 1.000 & 0.455 & 0.992 \\

\bottomrule
\end{tabular}
\end{table}

\subsection{Sequential Neural Network Simulations}
\label{neural_simulation}

\begin{table}[H]
\centering
\footnotesize
\setlength{\tabcolsep}{5pt}
\renewcommand{\arraystretch}{0.9}
\caption{LSTM simulation results}
\label{tab:neural_simulation_lstm}
\begin{tabular}{@{}cccccccccccc@{}}
\midrule
\multirow{2}{*}{\textbf{\makecell{Setting \\ (p*m, n)}}} &
\multirow{2}{*}{\textbf{\makecell{Signal \\ Strength}}} &
\multicolumn{2}{c}{\textbf{PSGM}} &
\multicolumn{2}{c}{\textbf{Fitted Derivative}} &
\multicolumn{2}{c}{\textbf{AGL}} &
\multicolumn{2}{c}{\textbf{AFS}} &
\multicolumn{2}{c}{\textbf{Sawaya}} \\
\cmidrule(lr){3-4} 
\cmidrule(lr){5-6} 
\cmidrule(lr){7-8} 
\cmidrule(lr){9-10} 
\cmidrule(lr){11-12}
& & FDR & Power & FDR & Power & FDR & Power & FDR & Power & FDR & Power \\
\midrule

\multirow{2}{*}{$\makecell{p*m = 1500 \\ n = 2500}$}
& Weak   & \textbf{0.161} & \textcolor{blue}{0.879} & \textcolor{blue}{0.213} & 0.842 & 0.916 & \textbf{0.966} & 0.886 & 0.556 & 0.448 & 0.007 \\
& Strong & \textbf{0.145} & \textbf{0.958} & \textcolor{blue}{0.265} & 0.790 & 0.833 & \textcolor{blue}{0.880} & 0.644 & 0.723 & 0.276 & 0.029 \\

\midrule

\multirow{2}{*}{$\makecell{p*m = 3500 \\ n = 2500}$}
& Weak   & \textbf{0.116} & 0.625 & \textcolor{blue}{0.146} & \textcolor{blue}{0.649} & 0.938 & \textbf{0.960} & 0.968 & 0.311 & 0.493 & 0.007 \\
& Strong & \textbf{0.152} & \textbf{0.886} & \textcolor{blue}{0.188} & 0.763 & 0.854 & \textcolor{blue}{0.880} & 0.912 & 0.683 & 0.519 & 0.019 \\

\bottomrule
\end{tabular}
\end{table}

\begin{table}[H]
\centering
\footnotesize
\setlength{\tabcolsep}{5pt}
\renewcommand{\arraystretch}{0.9}
\caption{GRU simulation results}
\label{tab:neural_simulation_gru}
\begin{tabular}{@{}cccccccccccc@{}}
\midrule
\multirow{2}{*}{\textbf{\makecell{Setting \\ (p*m, n)}}} &
\multirow{2}{*}{\textbf{\makecell{Signal \\ Strength}}} &
\multicolumn{2}{c}{\textbf{PSGM}} &
\multicolumn{2}{c}{\textbf{Fitted Derivative}} &
\multicolumn{2}{c}{\textbf{AGL}} &
\multicolumn{2}{c}{\textbf{AFS}} &
\multicolumn{2}{c}{\textbf{Sawaya}} \\
\cmidrule(lr){3-4} 
\cmidrule(lr){5-6} 
\cmidrule(lr){7-8} 
\cmidrule(lr){9-10} 
\cmidrule(lr){11-12}
& & FDR & Power & FDR & Power & FDR & Power & FDR & Power & FDR & Power \\
\midrule

\multirow{2}{*}{$\makecell{p*m = 1500 \\ n = 2500}$}
& Weak   & \textbf{0.133} & \textcolor{blue}{0.931} & \textcolor{blue}{0.239} & 0.821 & 0.837 & 0.899 & 0.555 & \textbf{0.944} & 0.476 & 0.176 \\
& Strong & \textbf{0.162} & \textcolor{blue}{0.979} & \textcolor{blue}{0.249} & 0.854 & 0.764 & 0.840 & 0.220 & \textbf{0.992} & 0.600 & 0.203 \\

\midrule

\multirow{2}{*}{$\makecell{p*m = 3500 \\ n = 2500}$}
& Weak   & \textbf{0.062} & 0.817 & \textcolor{blue}{0.134} & \textcolor{blue}{0.851} & 0.947 & \textbf{1.000} & 0.590 & 0.821 & 0.577 & 0.024 \\
& Strong & \textbf{0.223} & \textcolor{blue}{0.919} & \textcolor{blue}{0.308} & 0.889 & 0.829 & 0.880 & 0.441 & \textbf{1.000} & 0.478 & 0.163 \\

\bottomrule
\end{tabular}
\end{table}

\begin{table}[H]
\centering
\footnotesize
\setlength{\tabcolsep}{5pt}
\renewcommand{\arraystretch}{0.9}
\caption{Transformer simulation results}
\label{tab:neural_simulation_transformer}
\begin{tabular}{@{}cccccccccccc@{}}
\midrule
\multirow{2}{*}{\textbf{\makecell{Setting \\ (p*m, n)}}} &
\multirow{2}{*}{\textbf{\makecell{Signal \\ Strength}}} &
\multicolumn{2}{c}{\textbf{PSGM}} &
\multicolumn{2}{c}{\textbf{Fitted Derivative}} &
\multicolumn{2}{c}{\textbf{AGL}} &
\multicolumn{2}{c}{\textbf{AFS}} &
\multicolumn{2}{c}{\textbf{Sawaya}} \\
\cmidrule(lr){3-4} 
\cmidrule(lr){5-6} 
\cmidrule(lr){7-8} 
\cmidrule(lr){9-10} 
\cmidrule(lr){11-12}
& & FDR & Power & FDR & Power & FDR & Power & FDR & Power & FDR & Power \\
\midrule

\multirow{2}{*}{$\makecell{p*m = 1500 \\ n = 2500}$}
& Weak
& \textbf{0.068} & 0.846
& \textcolor{blue}{0.098} & 0.855
& 0.948 & \textbf{1.000}
& 0.668 & 0.591
& 0.224 & \textcolor{blue}{0.906} \\

& Strong
& \textbf{0.116} & \textcolor{blue}{0.887}
& \textcolor{blue}{0.232} & \textbf{0.893}
& 0.814 & 0.880
& 0.347 & 0.872
& 0.306 & 0.683 \\

\midrule

\multirow{2}{*}{$\makecell{p*m = 3500 \\ n = 2500}$}
& Weak
& \textbf{0.141} & 0.491
& \textcolor{blue}{0.183} & 0.577
& 0.974 & \textbf{1.000}
& 0.739 & 0.475
& 0.278 & \textcolor{blue}{0.667} \\

& Strong
& \textcolor{blue}{0.283} & \textcolor{blue}{0.900}
& \textbf{0.160} & 0.710
& 0.903 & \textbf{0.925}
& 0.720 & 0.843
& 0.432 & 0.619 \\

\bottomrule
\end{tabular}
\end{table}

We evaluate our Permutation SHAP Gaussian Mirror (PSGM) method on nonlinear sequential designs using LSTM, GRU, and Transformer encoder models. The target FDR level is still \(q=0.1\), and we fix \(| \mathcal S_1 |=15\) to mimic a sparse setting. As in linear simulations, each grouped feature has lag-block length consistent with the model's look-back window \(m\). We also consider a low-dimensional setting with \(n=2500\), \(p=300\), \(m=5\), and a high-dimensional setting with \(n=2500\), \(p=700\), \(m=5\). We also conduct runtime/memory complexity and performance sensitivity analysis of PSGM to different scales of \(p, n, m\) in both low- and high-dimensional settings in Appendix ~\ref{scexperiment}.

The covariates are generated from a latent factor design similar to \cite{Zuo2024}. At each time \(t\), we sample a 10-dimensional latent factor vector with correlated coordinates and smooth it over time:
\[
\boldsymbol f_t^{\rm raw}\sim N(\boldsymbol 0,\Sigma_f),
\quad
(\Sigma_f)_{ab}=0.9^{|a-b|},
\quad
\boldsymbol f_t=0.3\boldsymbol f_{t-1}^{\rm raw}+0.7\boldsymbol f_t^{\rm raw}.
\]
Given \(\boldsymbol f_t\), all features are generated by the same logistic factor mechanism. For each feature \(j\), we independently draw \(c_j\sim N(0,1)\) and \(\tilde{\boldsymbol\lambda}_j\sim N(\boldsymbol 0,I_{11})\). Let \(\tilde{\boldsymbol f}_t=(1,\boldsymbol f_t^\top)^\top\in\mathbb R^{11}\). We set
\[
X_{tj}
=
\frac{c_j}{1+\exp(\tilde{\boldsymbol\lambda}_j^\top \tilde{\boldsymbol f}_t)}
+
e_{tj},
\quad \text{where} \quad
e_{tj}\sim N(0,1).
\]
The relevant feature set is then sampled uniformly at random; all features are generated by the same mechanism, and relevance only determines whether a feature enters the response. For each relevant feature \(j\), we assign a signed signal weight \(\beta_j=10s_j\), where \(s_j\in\{-1,1\}\), while null features have \(\beta_j=0\). The response is generated from a lag-weighted nonlinear signal,
\[
\eta_t
=
\sum_{\ell=0}^{m-1} w_\ell\,\boldsymbol x_{t-1-\ell}^{\top}\boldsymbol\beta,
\qquad
y_t=\tanh(\eta_t)+\varepsilon_t,
\qquad
\varepsilon_t\sim N(0,\sigma_\varepsilon^2),
\]
where \(w_\ell\propto 0.9^\ell\) is normalized to have unit \(\ell_2\)-norm. We consider \(\sigma_\varepsilon\in\{1,0.3\}\), corresponding to weak and strong signal-to-noise regimes.

For each dataset, we train LSTM, GRU, and Transformer encoder models. The recurrent models use two layers, the Transformer uses two attention heads, and the hidden dimension is set to \(20\log(p)\). All models are trained with MSE loss using Adam, with learning rate, weight decay, and dropout selected by the same cross-validation procedure. We compare our algorithm using Permutation SHAP derivative with ablations using the fitted derivative at sub-feature for importance score. We also compare with three baselines: AFS \cite{Gui2019}, Adaptive Group Lasso (AGL) \cite{Dinh2020}, and Sawaya's data-splitting FDR-control variant \cite{Sawaya2025}. The first two are feature-selection heuristics and do not guarantee FDR control, while the third is designed for controlled neural feature selection. For each model, dimension setting, and noise level, we repeat the experiment 50 times, and report the averaged grouped-feature FDR and power. The results are summarized in Table~\ref{tab:neural_simulation_lstm}, Table~\ref{tab:neural_simulation_gru}, and Table~\ref{tab:neural_simulation_transformer}.

\section{Real Data Application: CausalRivers}
\label{realdata}

We evaluate our method on the East Germany subset of CausalRivers~\cite{Stein2025} dataset, a real-world river-discharge time-series benchmark with 666 gauge stations and a directed upstream-to-downstream river graph. We aggregate the original 15-minute measurements to 3-hour averages, interpolate missing values, and use first differences to reduce persistent level effects. For each target station \(j\), the real response is \(y_t^{\mathrm{real}}=\Delta R_{t,j}\), and each candidate station \(k\) forms one grouped lag feature \(X_{t,k}=(\Delta R_{t-1,k},\Delta R_{t-2,k},\Delta R_{t-3,k})\). Thus, selection is performed at the station-block level rather than the individual-lag level. We fix the target FDR level \(q=0.2\) for a less conservative selection.

\begin{table}[H]
\centering
\footnotesize
\setlength{\tabcolsep}{4pt}
\renewcommand{\arraystretch}{0.9}
\caption{LSTM results on the CausalRivers real data application}
\label{tab:causalrivers_lstm}
\begin{tabular}{@{}cccccccccc@{}}
\toprule
Outcome & Method 
& FDR/Proxy
& Power/Proxy
& \% Direct 
& \% 2nd  
& \% Null
& Sel. MSE 
& Full MSE 
& Oracle MSE \\
\midrule
\multirow{4}{*}{\makecell{Real \(y\)}}
& PSGM 
& \textcolor{blue}{0.739} & \textcolor{blue}{0.299} & 0.331 & 0.276 & 0.212 & \textbf{0.385} & 0.436 & 0.365 \\
& AGL 
& 0.772 & \textbf{1.000} & 1.000 & 1.000 & 1.000 & \textcolor{blue}{0.427} & 0.436 & 0.365 \\
& AFS 
& 0.795 & 0.140 & 0.131 & 0.150 & 0.155 & 0.598 & 0.436 & 0.365 \\
& Sawaya 
& \textbf{0.405} & 0.029 & 0.036 & 0.017 & 0.020 & 0.613 & 0.436 & 0.365 \\
\midrule
\multirow{4}{*}{\makecell{Semi-\\synthetic \(y\)}}
& PSGM 
& \textbf{0.306} & \textcolor{blue}{0.896} & 0.992 & 0.826 & 0.143 & \textbf{0.512} & 0.575 & 0.501 \\
& AGL 
& 0.772 & \textbf{1.000} & 1.000 & 1.000 & 1.000 & \textcolor{blue}{0.592} & 0.575 & 0.501 \\
& AFS 
& 0.919 & 0.056 & 0.064 & 0.059 & 0.178 & 0.865 & 0.575 & 0.501 \\
& Sawaya 
& \textcolor{blue}{0.342} & 0.040 & 0.044 & 0.025 & 0.023 & 0.837 & 0.575 & 0.501 \\
\bottomrule
\end{tabular}
\end{table}

\begin{table}[H]
\centering
\footnotesize
\setlength{\tabcolsep}{4pt}
\renewcommand{\arraystretch}{0.9}
\caption{GRU results on the CausalRivers real data application}
\label{tab:causalrivers_gru}
\begin{tabular}{@{}cccccccccc@{}}
\toprule
Outcome & Method 
& FDR/Proxy
& Power/Proxy
& \% Direct 
& \% 2nd  
& \% Null
& Sel. MSE 
& Full MSE 
& Oracle MSE \\
\midrule
\multirow{4}{*}{\makecell{Real \(y\)}}
& PSGM 
& \textcolor{blue}{0.700} & \textcolor{blue}{0.379} & 0.403 & 0.367 & 0.245 & \textbf{0.373} & 0.439 & 0.378 \\
& AGL 
& 0.772 & \textbf{1.000} & 1.000 & 1.000 & 1.000 & \textcolor{blue}{0.437} & 0.439 & 0.378 \\
& AFS 
& 0.757 & 0.162 & 0.172 & 0.156 & 0.142 & 0.650 & 0.439 & 0.378 \\
& Sawaya 
& \textbf{0.458} & 0.057 & 0.051 & 0.062 & 0.022 & 0.513 & 0.439 & 0.378 \\
\midrule
\multirow{4}{*}{\makecell{Semi-\\synthetic \(y\)}}
& PSGM 
& \textbf{0.264} & \textcolor{blue}{0.902} & 0.983 & 0.854 & 0.110 & \textbf{0.508} & 0.592 & 0.523 \\
& AGL 
& 0.772 & \textbf{1.000} & 1.000 & 1.000 & 1.000 & \textcolor{blue}{0.587} & 0.592 & 0.523 \\
& AFS 
& 0.916 & 0.068 & 0.111 & 0.045 & 0.198 & 0.841 & 0.592 & 0.523 \\
& Sawaya 
& \textcolor{blue}{0.347} & 0.052 & 0.078 & 0.043 & 0.017 & 0.819 & 0.592 & 0.523 \\
\bottomrule
\end{tabular}
\end{table}

\begin{table}[H]
\centering
\footnotesize
\setlength{\tabcolsep}{4pt}
\renewcommand{\arraystretch}{0.9}
\caption{Transformer results on the CausalRivers real data application}
\label{tab:causalrivers_transformer}
\begin{tabular}{@{}cccccccccc@{}}
\toprule
Outcome & Method 
& FDR/Proxy
& Power/Proxy
& \% Direct 
& \% 2nd  
& \% Null
& Sel. MSE 
& Full MSE 
& Oracle MSE \\
\midrule
\multirow{4}{*}{\makecell{Real \(y\)}}
& PSGM 
& \textcolor{blue}{0.576} & \textcolor{blue}{0.406} & 0.431 & 0.407 & 0.155 & \textcolor{blue}{0.333} & 0.399 & 0.345 \\
& AGL 
& 0.772 & \textbf{1.000} & 1.000 & 1.000 & 1.000 & 0.407 & 0.399 & 0.345 \\
& AFS 
& 0.616 & 0.219 & 0.350 & 0.132 & 0.122 & \textbf{0.332} & 0.399 & 0.345 \\
& Sawaya 
& \textbf{0.307} & 0.092 & 0.151 & 0.048 & 0.028 & 0.382 & 0.399 & 0.345 \\
\midrule
\multirow{4}{*}{\makecell{Semi-\\synthetic \(y\)}}
& PSGM 
& \textbf{0.258} & \textcolor{blue}{0.897} & 0.992 & 0.825 & 0.115 & \textbf{0.514} & 0.586 & 0.509 \\
& AGL 
& 0.772 & \textbf{1.000} & 1.000 & 1.000 & 1.000 & 0.597 & 0.586 & 0.509 \\
& AFS 
& \textcolor{blue}{0.320} & 0.452 & 0.763 & 0.249 & 0.063 & \textcolor{blue}{0.593} & 0.586 & 0.509 \\
& Sawaya 
& 0.340 & 0.124 & 0.233 & 0.035 & 0.023 & 0.720 & 0.586 & 0.509 \\
\bottomrule
\end{tabular}
\end{table}

We use the first 2400 differenced observations after January 1, 2021 with a train-test ratio of \(4:1\). We consider six target stations selected by a fixed rule requiring at least 10 graph-relevant groups, 50 eligible graph-null groups, low-dimensional OLS feasibility, and acceptable condition numbers. Direct upstream and upstream-of-upstream stations are treated as graph-proxy relevant groups, while graph-null groups are sampled from stations with no directed path relation to the target. Since the river graph does not define the statistical null in our regression model, FDR and power on real \(y\) are reported only as graph-proxy metrics. To evaluate selection with known support under real covariate dependence, we also generate semi-synthetic responses on the same river covariates, using direct and second-order upstream groups as the true support. For each relevant group \(k\), let \(s_{t,k}=\boldsymbol{w}^\top X_{t,k}\), where \(\boldsymbol{w}\) is an \(\ell_2\)-normalized exponentially decaying lag-weight vector with larger weight on recent lags. We generate
\[
\eta_t
=
\sum_{k\in \mathcal S_{\mathrm{direct}}}\xi_k \tanh(s_{t,k})
+
0.55\sum_{k\in \mathcal S_{\mathrm{second}}}\xi_k \sin(s_{t,k})
+
0.20\sum_{(k,k')\in \mathcal P_{\mathrm{direct}}}
\tanh(s_{t,k}s_{t,k'}/2),
\]

where \(\xi_k\in\{-1,1\}\) and \(\mathcal P_{\mathrm{direct}}\) contains pairs of direct upstream groups. We set \(y_t^{\mathrm{syn}}=\eta_t+\varepsilon_t\), with Gaussian noise calibrated to signal-to-noise ratio \(2.0\). For semi-synthetic \(y\), FDR and power are computed using the constructed true support. Table ~\ref{tab:causalrivers_lstm}, ~\ref{tab:causalrivers_gru}, and ~\ref{tab:causalrivers_transformer} reports LSTM, GRU and Transformer results averaged over six target stations and ten random draws of graph-null groups, including FDR/proxy FDR, power/proxy power, selected-group composition, and held-out test MSE for the selected, full, and oracle upstream-only models. Full dataset description, preprocessing details, target-level results, and linear model results are given in Appendix~\ref{causalriverdetails}.

\section{Conclusion and Limitations}
\label{conclusion}

We propose a model-agnostic FDR-control framework for grouped-feature selection in sequential models. To the best of our knowledge, this is the first Gaussian-Mirror-based method that extends mirror-statistic FDR control to sequential deep learning architectures while directly targeting grouped-feature null hypotheses. We introduce GGM for linear models and PSGM for neural networks. Across simulations and the CausalRivers application, our method achieves stronger FDR-power tradeoffs than benchmark methods while maintaining competitive predictive performance.

There are several limitations of our method. The neural version depends on the quality of the fitted predictive model, or how the fitted predictive model uses its input features. Null symmetry of PSGM mirror statistics may be weakened if the network learns spurious correlations or assigns systematic directional importance to null grouped features. This issue is closely related to the broader challenges of predictive and explanation multiplicity \cite{3586589.3586815, fisher, Geirhos2020, laberge2023partial, Ghorbani2017}. We discuss this issue further in Appendix ~\ref{multiplicity} and identify it as an important direction for future research. In our experiments, stronger regularization, particularly greater weight decay, generally improves FDR control.  PSGM also performs best when the grouped-feature size \(m\) is moderate and the relevant support is sparse; larger lag windows or denser signals may weaken mirror-statistic separation. Finally, our current method is relatively power-oriented and can sometimes be less conservative in FDR control, so designing more conservative neural mirror procedures is another useful future direction.

\clearpage

\begin{ack}
\textbf{Acknowledgments.} The authors thank Columbia Business School for providing access to computational resources used in this research.

\textbf{Funding.} This research received no external funding.

\textbf{Competing Interests.} The authors declare no competing interests.
\end{ack}

\bibliography{references}
\nocite{*}

\clearpage

\appendix

\section*{Appendix}

\section{Algorithms for GGM and PSGM}
\label{algorithms}

\begin{table}[H]
\centering
\begin{tabular}{@{}p{0.95\textwidth}@{}}
\toprule
\textbf{Algorithm 1} Low-Dimensional Group Gaussian Mirror \\
\midrule

\textbf{Input:} Fixed FDR level \(q\), grouped design
\(X=[X_1,\ldots,X_p]\) with \(X_j\in\mathbb R^{n\times m}\), and \(\boldsymbol y\in\mathbb R^n\). \\

\textbf{Output:} Selected grouped feature set \(\widehat{\mathcal S}_1\).

\begin{enumerate}[leftmargin=18pt,label={(\alph*)}]
    \item \textbf{for} \(j=1,\ldots,p\), \textbf{do}:
    \begin{enumerate}[leftmargin=20pt,label={(\arabic*)}]
        \item Let \(X_{-j}\) be the design matrix with group \(j\) removed, and compute \(R_j=M_{X_{-j}}X_j\).

        \item Generate an auxiliary Gaussian noise matrix \(Z_j\in\mathbb R^{n\times m}\), with \(Z_j\perp (X,\boldsymbol y)\).

        \item Orthogonalize and normalize the noise block to get \(Q_j\) and \(\widetilde Z_j\).

        \item Set the group perturbation matrix \(G_j=(R_j^\top R_j)^{1/2}\).

        \item Construct the block-level mirror variables \(X_j^+=X_j+\widetilde Z_jG_j\) and \(X_j^-=X_j-\widetilde Z_jG_j\).

        \item Fit the linear regression of \(\boldsymbol y\) on \([X_{-j},X_j^+,X_j^-]\) and obtain \(\widehat{\boldsymbol\beta}_j^\pm\).

        \item Compute the block-level mirror statistic \(M_j=\operatorname{sign}\left(\langle \widehat{\boldsymbol\beta}_j^+,\widehat{\boldsymbol\beta}_j^- \rangle\right)\left(\|\widehat{\boldsymbol\beta}_j^+\|_2+\|\widehat{\boldsymbol\beta}_j^-\|_2\right)\).
    \end{enumerate}
    \item \textbf{end for}

    \item Compute the mirror threshold \(\tau_q=\min\left\{t>0:\frac{\#\{j:M_j\le -t\}+1}{\#\{j:M_j\ge t\}\vee 1}\le q\right\}\).

    \item Return \(\widehat{\mathcal S}_1=\{j:M_j\ge \tau_q\}\).
\end{enumerate}
\\
\bottomrule
\end{tabular}
\end{table}

\begin{table}[H]
\centering
\begin{tabular}{@{}p{0.95\textwidth}@{}}
\toprule
\textbf{Algorithm 2} High-Dimensional Two-Stage Group Gaussian Mirror \\
\midrule

\textbf{Input:} Fixed FDR level \(q\), grouped design
\(X=[X_1,\ldots,X_p]\) with \(X_j\in\mathbb R^{n\times m}\), response
\(\boldsymbol y\in\mathbb R^n\), and first-stage grouped screener. \\

\textbf{Output:} Selected grouped feature set \(\widehat{\mathcal S}_1\).

\begin{enumerate}[leftmargin=18pt,label={(\alph*)}]
    \item Split the data into two chunks \((X^{(1)},\boldsymbol y^{(1)})\) and \((X^{(2)},\boldsymbol y^{(2)})\). For time-series data, use two consecutive chunks and leave a short buffer gap if needed.

    \item Apply a grouped feature screening method (Group Lasso) on the first chunk and obtain \(\widehat{\mathcal A}\subseteq\{1,\ldots,p\}\), and initialize \(\widetilde M_j=0\), for \(j=1,\ldots,p\).

    \item \textbf{for} \(j\in\widehat{\mathcal A}\), \textbf{do}:
    \begin{enumerate}[leftmargin=20pt,label={(\arabic*)}]
        \item Restrict the second-stage design to the screened model \(X_{\widehat{\mathcal A}}^{(2)}=\{X_k^{(2)}:k\in\widehat{\mathcal A}\}\).

        \item Compute the residualized group block: \(R_j^{(2)}=M_{X_{\widehat{\mathcal A}\setminus\{j\}}^{(2)}}X_j^{(2)}\).

        \item Generate a Gaussian noise matrix \(Z_j^{(2)}\in\mathbb R^{n_2\times m}\), with \(Z_j^{(2)}\perp (X^{(2)},\boldsymbol y^{(2)})\).

        \item Orthogonalize and normalize the noise block to get \(Q_j^{(2)}\), \(\widetilde Z_j^{(2)}\) and \(G_j^{(2)}\).

        \item Construct second-stage mirror variables \(X_j^{\pm,(2)}=X_j^{(2)}\pm\widetilde Z_j^{(2)}G_j^{(2)}\).

        \item Regress \(\boldsymbol y^{(2)}\) on \([X_{\widehat{\mathcal A}\setminus\{j\}}^{(2)},X_j^{+,(2)},X_j^{-,(2)}]\) and obtain \(\widehat{\boldsymbol\beta}_j^{\pm,(2)}\).

        \item Compute \( \widetilde M_j=M_j^{(2)}=\operatorname{sign}\left(\langle \widehat{\boldsymbol\beta}_j^{+,(2)},\widehat{\boldsymbol\beta}_j^{-,(2)} \rangle\right)\left(\|\widehat{\boldsymbol\beta}_j^{+,(2)}\|_2+\|\widehat{\boldsymbol\beta}_j^{-,(2)}\|_2\right)\).

    \end{enumerate}
    \item \textbf{end for}

    \item Compute the mirror threshold \(\tau_q=\min\left\{t>0:\frac{\#\{j:\widetilde M_j\le -t\}+1}{\#\{j:\widetilde M_j\ge t\}\vee 1}\le q\right\}\).

    \item Return \(\widehat{\mathcal S}_1=\{j:\widetilde M_j\ge \tau_q\}\).
\end{enumerate}
\\
\bottomrule
\end{tabular}
\end{table}

\begin{table}[H]
\centering
\begin{tabular}{@{}p{0.95\textwidth}@{}}
\toprule
\textbf{Algorithm 3} Low-Dimensional Permutation SHAP Gaussian Mirror \\
\midrule

\textbf{Input:} Fixed FDR level \(q\), grouped design
\(X=[X_1,\ldots,X_p]\) with \(X_j\in\mathbb R^{n\times m}\), response \(\boldsymbol y\in\mathbb R^n\), model class \(\mathcal F\), kernels \(k_U,k_V,k_W\), permutation draws \(M_n\), smoothing bandwidth \(h_n\). \\

\textbf{Output:} Selected grouped feature set \(\widehat{\mathcal S}_1\).

\begin{enumerate}[leftmargin=18pt,label={(\alph*)}]
    \item \textbf{for} \(j=1,\ldots,p\), \textbf{do}:
    \begin{enumerate}[leftmargin=20pt,label={(\arabic*)}]

        \item Construct an orthogonalized Gaussian noise block \(\widetilde Z_j\in\mathbb R^{n\times m}\), with \(\widetilde Z_j\perp (X,\boldsymbol y)\).

        \item For each candidate \(G_j\in\mathcal G\), construct \(U_j(G_j), V_j(G_j)=X_j\pm\widetilde Z_jG_j\).

        \item Optimize the perturbation matrix and get \(\widehat G_j\in\arg\min_{G_j\in\mathcal G} I_j^K(G_j)^2\).

        \item Construct the final neural mirror branches \(U_j=X_j+\widetilde Z_j\widehat G_j\) and \(V_j=X_j-\widetilde Z_j\widehat G_j\).

    \end{enumerate}
    \item \textbf{end for}

    \item Train a neural model \(\widehat f\in\mathcal F\) using the augmented input \([U_1,V_1,\ldots,U_p,V_p]\) and \(\boldsymbol y\).

    \item \textbf{for} \(j=1,\ldots,p\), \textbf{do}:
    \begin{enumerate}[leftmargin=20pt,label={(\arabic*)}]
        \item For each sub-feature \(k\), estimate the local-quadratic smoothed Permutation SHAP derivatives \(\widehat\psi_{j,k}^\pm(\boldsymbol x_t)\) for each \(t\). Aggregate them and get \(\widehat{\boldsymbol\psi}_j^\pm=\frac{1}{n}\sum_{t=1}^n\widehat{\boldsymbol\psi}_j^\pm(\boldsymbol x_t)\).

        \item Compute the neural mirror statistic \(M_j=\operatorname{sign}\left(\langle \widehat{\boldsymbol\psi}_j^+,\widehat{\boldsymbol\psi}_j^- \rangle\right)\left(\|\widehat{\boldsymbol\psi}_j^+\|_2+\|\widehat{\boldsymbol\psi}_j^-\|_2\right)\).
    \end{enumerate}
    \item \textbf{end for}

    \item Compute the mirror threshold \(\tau_q=\min\left\{t>0:\frac{\#\{j:M_j\le -t\}+1}{\#\{j:M_j\ge t\}\vee 1}\le q\right\}\).

    \item Return \(\widehat{\mathcal S}_1=\{j:M_j\ge \tau_q\}\).
\end{enumerate}
\\
\bottomrule
\end{tabular}
\end{table}

\begin{table}[H]
\centering
\begin{tabular}{@{}p{0.95\textwidth}@{}}
\toprule
\textbf{Algorithm 4} High-Dimensional Ridge-Projected Permutation SHAP Gaussian Mirror \\
\midrule

\textbf{Input:} Fixed FDR level \(q\), grouped design
\(X=[X_1,\ldots,X_p]\) with \(X_j\in\mathbb R^{n\times m}\), response \(\boldsymbol y\in\mathbb R^n\), model class \(\mathcal F\), kernels \(k_U,k_V,k_W\), permutation draws \(M_n\), smoothing bandwidth \(h_n\). \\

\textbf{Output:} Selected grouped feature set \(\widehat{\mathcal S}_1\).

\begin{enumerate}[leftmargin=18pt,label={(\alph*)}]
    \item \textbf{for} \(j=1,\ldots,p\), \textbf{do}:
    \begin{enumerate}[leftmargin=20pt,label={(\arabic*)}]

        \item Generate Gaussian noise matrix \(Z_j\in\mathbb R^{n\times m}\), with \(Z_j\overset d=-Z_j\) and \(Z_j\perp (X,\boldsymbol y)\).

        \item Construct a label-free ridge-type projection operator \(A_{j,\lambda}(X_j,W_j)\in\mathbb R^{n\times n}\).

        \item Form the ridge-projected noise block \(Q_{j,\lambda}\), normalize it and get \(\widetilde Z_{j,\lambda}\).

        \item For each candidate \(G_j\in\mathcal G\), construct \(U_j(G_j), V_j(G_j)=X_j\pm\widetilde Z_{j,\lambda}G_j\).

        \item Initialize the perturbation matrix \(G_j^{(0)}=I_m\), optimize \(\widehat G_j\in\arg\min_{G_j\in\mathcal G} I_j^K(G_j)^2\).

        \item Construct the final mirror branches \(U_j=X_j+\widetilde Z_{j,\lambda}\widehat G_j\) and \(V_j=X_j-\widetilde Z_{j,\lambda}\widehat G_j\).

    \end{enumerate}
    \item \textbf{end for}

    \item Train a neural model \(\widehat f\in\mathcal F\) using the augmented input \([U_1,V_1,\ldots,U_p,V_p]\) and \(\boldsymbol y\).

    \item \textbf{for} \(j=1,\ldots,p\), \textbf{do}:
    \begin{enumerate}[leftmargin=20pt,label={(\arabic*)}]

        \item For each sub-feature \(k\), estimate the local-quadratic smoothed Permutation SHAP derivatives \(\widehat\psi_{j,k}^\pm(\boldsymbol x_t)\) for each \(t\). Aggregate them and get \(\widehat{\boldsymbol\psi}_j^\pm=\frac{1}{n}\sum_{t=1}^n\widehat{\boldsymbol\psi}_j^\pm(\boldsymbol x_t)\).

        \item Compute the neural mirror statistic \(M_j=\operatorname{sign}\left(\langle \widehat{\boldsymbol\psi}_j^+,\widehat{\boldsymbol\psi}_j^- \rangle\right)\left(\|\widehat{\boldsymbol\psi}_j^+\|_2+\|\widehat{\boldsymbol\psi}_j^-\|_2\right)\).

    \end{enumerate}
    \item \textbf{end for}

    \item Compute the mirror threshold \(\tau_q=\min\left\{t>0:\frac{\#\{j:M_j\le -t\}+1}{\#\{j:M_j\ge t\}\vee 1}\le q\right\}\).

    \item Return \(\widehat{\mathcal S}_1=\{j:M_j\ge \tau_q\}\).
\end{enumerate}
\\
\bottomrule
\end{tabular}
\end{table}

\section{Flowchart of GGM and PSGM}
\label{flowchart}

\begin{figure}[H]
    \centering
    \includegraphics[width=1\textwidth]{"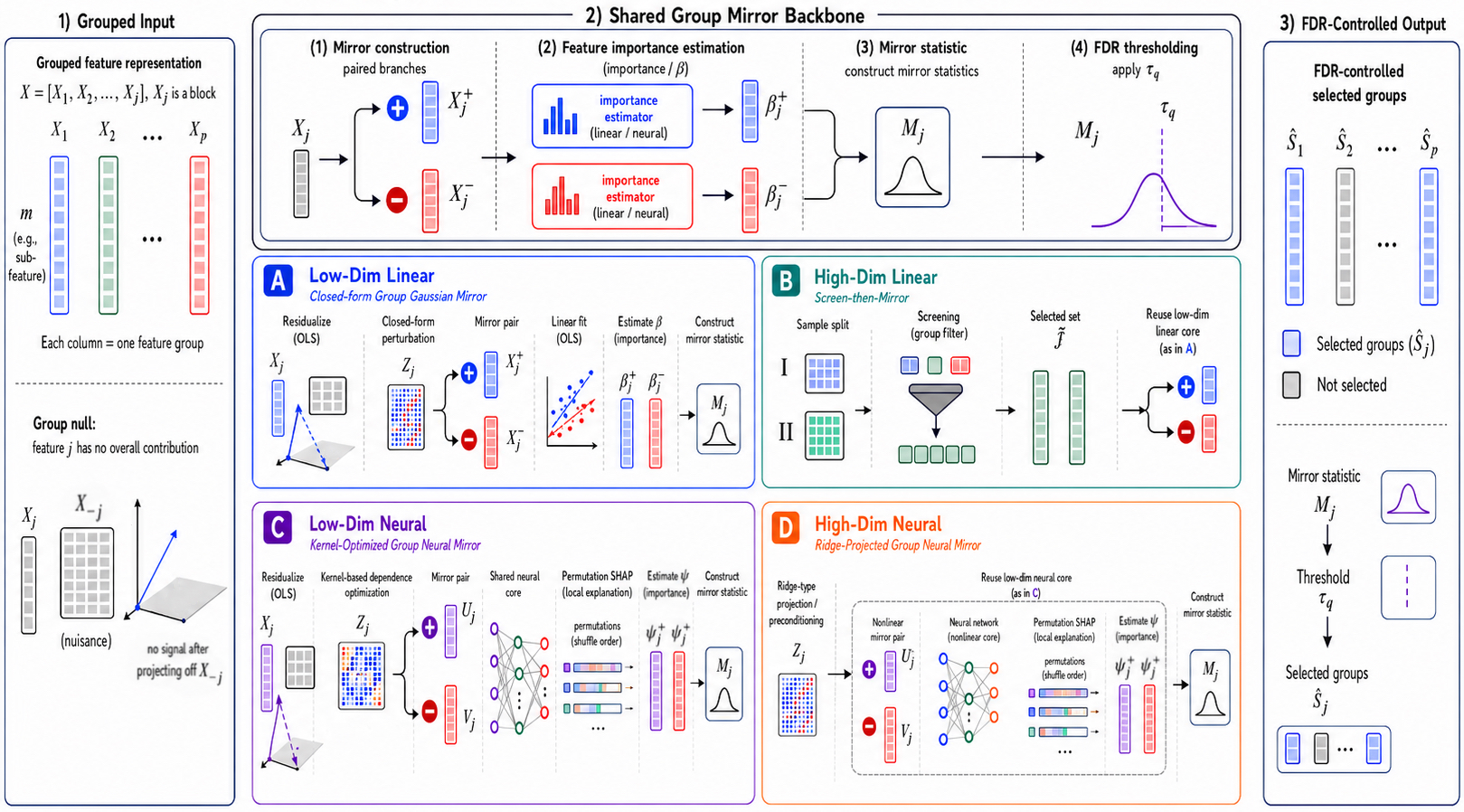"}
    \caption{Flowchart of GGM and PSGM}
    \label{main_graph}
\end{figure}

\section{Proof of Theorems and Propositions}
\label{proofs}

\subsection{Proof of Proposition \ref{prop:gm_fail_grouped}}

\gmfailgrouped*

\begin{proof}
In the grouped design, the \(j\)-th original feature is represented by the
block
\[
X_j=[x_{j,0},x_{j,1},\ldots,x_{j,m-1}]\in\mathbb R^{n\times m},
\]
with coefficient vector \(\boldsymbol\beta_j = (\beta_{j,0},\beta_{j,1},\ldots,\beta_{j,m-1})^\top
\in\mathbb R^m\). Therefore, the feature-level null hypothesis for group \(j\) is \(H_{0,j}:\boldsymbol\beta_j=\boldsymbol 0\), which is equivalent to the joint condition
\[
\beta_{j,0}=\beta_{j,1}=\cdots=\beta_{j,m-1}=0.
\]
If the original Gaussian Mirror is applied separately to each sub-feature \(x_{j,\ell}\) of \(X_j\), then each mirror
construction is associated with a single scalar coefficient \(\beta_{j,\ell}\). Thus the null hypothesis targeted by the \(\ell\)-th coordinate-wise mirror construction is
\[
H_{0,(j,\ell)}:\beta_{j,\ell}=0,
\qquad \ell=0,\ldots,m-1 .
\]
This coordinate-wise null hypothesis is not the same inferential object as the grouped null \(H_{0,j}:\boldsymbol\beta_j=\boldsymbol 0\), which concerns
the simultaneous vanishing of all coefficients in the block.

Consequently, applying the original Gaussian Mirror coordinate by coordinate
targets the collection of coordinate-wise hypotheses
\[
\{H_{0,(j,\ell)}:\ell=0,\ldots,m-1\},
\]
rather than the grouped feature-level null \(H_{0,j}\). This completes the proof.
\end{proof}

\subsection{Proof of Theorem \ref{thm:low_dim_ggm}}

\lowdimggm*

\begin{proof}
Fix a group \(j\). We have \(R_j=M_{X_{-j}}X_j\), \(Q_j=(I_n-P_{[X_{-j},R_j]})Z_j\), \(\widetilde Z_j=Q_j(Q_j^\top Q_j)^{-1/2}\) and \(G_j=(R_j^\top R_j)^{1/2}\). The nonsingularity of \(R_j^\top R_j\) and \(Q_j^\top Q_j\) guarantees that \(G_j\) and \(\widetilde Z_j\) are well-defined. Since \(Q_j\) is obtained by projecting \(Z_j\) onto the orthogonal complement of the column space of \([X_{-j},R_j]\), we have
\[
Q_j^\top X_{-j}=0,
\qquad
Q_j^\top R_j=0 .
\]
By the definition of \(\widetilde Z_j\), we have
\[
\widetilde Z_j^\top \widetilde Z_j
=
(Q_j^\top Q_j)^{-1/2}Q_j^\top Q_j(Q_j^\top Q_j)^{-1/2}
=
I_m .
\]
\[
\widetilde Z_j^\top X_{-j}=0,
\qquad
\widetilde Z_j^\top R_j=0 .
\]
Define the residualized mirror blocks \(V_j^\pm=M_{X_{-j}}X_j^\pm\). Since \(X_j^\pm=X_j\pm\widetilde Z_jG_j\) and \(\widetilde Z_j\) is orthogonal to the column space of \(X_{-j}\), we have
\[
M_{X_{-j}}\widetilde Z_j=\widetilde Z_j.
\]
Therefore,
\[
V_j^\pm
=
M_{X_{-j}}X_j^\pm
=
R_j\pm\widetilde Z_jG_j,
\]
We first establish the key block orthogonality. Since \(G_j=(R_j^\top R_j)^{1/2}\), we have \(G_j^\top=G_j\) and \(G_j^2=R_j^\top R_j\). Using \(R_j^\top\widetilde Z_j=0\) and
\(\widetilde Z_j^\top\widetilde Z_j=I_m\), we obtain
\[
\begin{aligned}
(V_j^+)^\top V_j^-
&=
(R_j+\widetilde Z_jG_j)^\top
(R_j-\widetilde Z_jG_j) \\
&=
R_j^\top R_j
-
R_j^\top \widetilde Z_jG_j
+
G_j\widetilde Z_j^\top R_j
-
G_j\widetilde Z_j^\top \widetilde Z_jG_j \\
&=
R_j^\top R_j-G_j^2 \\
&=
0_{m\times m}.
\end{aligned}
\]
Similarly,
\[
\begin{aligned}
(V_j^+)^\top V_j^+
&=
(R_j+\widetilde Z_jG_j)^\top
(R_j+\widetilde Z_jG_j) \\
&=
R_j^\top R_j+G_j^2 \\
&=
2R_j^\top R_j,
\end{aligned}
\]
and
\[
(V_j^-)^\top V_j^-=2R_j^\top R_j .
\]
Now consider the grouped null \(H_{0,j}:\boldsymbol\beta_j=\boldsymbol 0\). Under this null, the grouped linear model can be written as
\[
Y=X_{-j}\boldsymbol\beta_{-j}+\varepsilon,
\qquad
\varepsilon\sim N(0,\sigma^2 I_n).
\]
Residualizing with respect to \(X_{-j}\) gives
\[
Y^*
:=
M_{X_{-j}}Y
=
M_{X_{-j}}\varepsilon
=:\varepsilon^* .
\]
Hence \(\mathbb E[\varepsilon^*]=0\) and \(\operatorname{Cov}(\varepsilon^*)=\sigma^2 M_{X_{-j}}\). By the Frisch--Waugh--Lovell theorem, the fitted coefficients of
\(X_j^+\) and \(X_j^-\) in the regression on
\([X_{-j},X_j^+,X_j^-]\) are equal to the fitted coefficients from regressing \(Y^*\) on the residualized blocks \([V_j^+,V_j^-]\). Since
\[
(V_j^+)^\top V_j^-=0_{m\times m},
\]
the two residualized mirror blocks are block-orthogonal. Therefore,
\[
\widehat{\boldsymbol\beta}_j^+
=
\bigl((V_j^+)^\top V_j^+\bigr)^{-1}(V_j^+)^\top Y^*
=
(2R_j^\top R_j)^{-1}(V_j^+)^\top \varepsilon^*,
\]
and
\[
\widehat{\boldsymbol\beta}_j^-
=
\bigl((V_j^-)^\top V_j^-\bigr)^{-1}(V_j^-)^\top Y^*
=
(2R_j^\top R_j)^{-1}(V_j^-)^\top \varepsilon^* .
\]

Since \(\mathbb E[\varepsilon^*]=0\), it follows that
\[
\mathbb E[\widehat{\boldsymbol\beta}_j^+]
=
\mathbb E[\widehat{\boldsymbol\beta}_j^-]
=
\boldsymbol 0 .
\]
Moreover, \(V_j^+\) and \(V_j^-\) lie in the residual space of \(X_{-j}\). Hence \(M_{X_{-j}}V_j^\pm=V_j^\pm\). Thus,
\[
\begin{aligned}
\operatorname{Cov}(\widehat{\boldsymbol\beta}_j^+)
&=
(2R_j^\top R_j)^{-1}
(V_j^+)^\top
\operatorname{Cov}(\varepsilon^*)
V_j^+
(2R_j^\top R_j)^{-1} \\
&=
\sigma^2
(2R_j^\top R_j)^{-1}
(V_j^+)^\top
M_{X_{-j}}
V_j^+
(2R_j^\top R_j)^{-1} \\
&=
\sigma^2
(2R_j^\top R_j)^{-1}
(V_j^+)^\top
V_j^+
(2R_j^\top R_j)^{-1} \\
&=
\sigma^2
(2R_j^\top R_j)^{-1}
(2R_j^\top R_j)
(2R_j^\top R_j)^{-1} \\
&=
\sigma^2(2R_j^\top R_j)^{-1}.
\end{aligned}
\]
The same calculation gives
\[
\operatorname{Cov}(\widehat{\boldsymbol\beta}_j^-)
=
\sigma^2(2R_j^\top R_j)^{-1}.
\]
Therefore,
\[
\operatorname{Cov}(\widehat{\boldsymbol\beta}_j^+)
=
\operatorname{Cov}(\widehat{\boldsymbol\beta}_j^-)
=
\sigma^2(2R_j^\top R_j)^{-1}.
\]
For the cross-covariance, we have
\[
\begin{aligned}
\operatorname{Cov}
(\widehat{\boldsymbol\beta}_j^+,\widehat{\boldsymbol\beta}_j^-)
&=
(2R_j^\top R_j)^{-1}
(V_j^+)^\top
\operatorname{Cov}(\varepsilon^*)
V_j^-
(2R_j^\top R_j)^{-1} \\
&=
\sigma^2
(2R_j^\top R_j)^{-1}
(V_j^+)^\top
M_{X_{-j}}
V_j^-
(2R_j^\top R_j)^{-1} \\
&=
\sigma^2
(2R_j^\top R_j)^{-1}
(V_j^+)^\top
V_j^-
(2R_j^\top R_j)^{-1} \\
&=
0_{m\times m}.
\end{aligned}
\]
Furthermore, \(\widehat{\boldsymbol\beta}_j^+\) and
\(\widehat{\boldsymbol\beta}_j^-\) are linear transformations of the Gaussian noise vector \(\varepsilon\). Hence they are jointly Gaussian. Since jointly Gaussian random vectors with zero cross-covariance are independent, we conclude that \(\widehat{\boldsymbol\beta}_j^+\) and \(\widehat{\boldsymbol\beta}_j^-\) are independent centered Gaussian vectors with identical covariance.

We now verify the null symmetry of the block-level mirror statistic. Recall that
\[
M_j
=
\operatorname{sign}\!\left(
\langle
\widehat{\boldsymbol\beta}_j^+,
\widehat{\boldsymbol\beta}_j^-
\rangle
\right)
\left(
\|\widehat{\boldsymbol\beta}_j^+\|_2+
\|\widehat{\boldsymbol\beta}_j^-\|_2
\right).
\]
Since \(\widehat{\boldsymbol\beta}_j^-\) is a centered Gaussian vector, \(\widehat{\boldsymbol\beta}_j^-\overset{d}{=}-\widehat{\boldsymbol\beta}_j^-\). Together with independence, this gives
\[
(\widehat{\boldsymbol\beta}_j^+,\widehat{\boldsymbol\beta}_j^-)
\overset{d}{=}
(\widehat{\boldsymbol\beta}_j^+,-\widehat{\boldsymbol\beta}_j^-).
\]
We define \(\mathcal M(a,b)=\operatorname{sign}(\langle a,b\rangle)(\|a\|_2+\|b\|_2)\), then \(\mathcal M(a,-b)
=
\operatorname{sign}(-\langle a,b\rangle)
(\|a\|_2+\|b\|_2)
=
-\mathcal M(a,b)\), with the convention \(\operatorname{sign}(0)=0\). Therefore,
\[
M_j
=
\mathcal M(\widehat{\boldsymbol\beta}_j^+,
\widehat{\boldsymbol\beta}_j^-)
\overset{d}{=}
\mathcal M(\widehat{\boldsymbol\beta}_j^+,
-\widehat{\boldsymbol\beta}_j^-)
=
-M_j .
\]
Thus \(M_j\) is symmetric about zero under \(H_{0,j}\). Now, it only remains to verify the alternative mean statement. Under \(H_{1,j}:\boldsymbol\beta_j\neq \boldsymbol 0\), the model can be written as
\[
Y=X_{-j}\boldsymbol\beta_{-j}+X_j\boldsymbol\beta_j+\varepsilon.
\]
\[
Y^*
=
M_{X_{-j}}Y
=
R_j\boldsymbol\beta_j+\varepsilon^* .
\]
Using \(V_j^\pm=R_j\pm\widetilde Z_jG_j\), we obtain
\[
R_j=\frac{1}{2}(V_j^+ + V_j^-).
\]
\[
R_j\boldsymbol\beta_j
=
V_j^+\frac{\boldsymbol\beta_j}{2}
+
V_j^-\frac{\boldsymbol\beta_j}{2}.
\]
Thus, in the residualized regression of \(Y^*\) on \([V_j^+,V_j^-]\), the population coefficients on the two mirror blocks are both \(\boldsymbol\beta_j/2\). Since the noise has mean zero, the OLS fitted coefficient vectors satisfy
\[
\mathbb E[\widehat{\boldsymbol\beta}_j^+]
=
\mathbb E[\widehat{\boldsymbol\beta}_j^-]
=
\frac{1}{2}\boldsymbol\beta_j .
\]
This completes the proof.
\end{proof}

\subsection{Proof of Proposition \ref{prop:full}}

\fullmethod*

\begin{proof}
Let \(S_0\) and \(S_1\) denote the sets of null and non-null grouped features, respectively. For any threshold \(t>0\), define
\[
R(t)=\#\{j:M_j\ge t\}, \quad V(t)=\#\{j\in S_0:M_j\ge t\},
\]
\[
V^{-}(t)=\#\{j\in S_0:M_j\le -t\}, \quad N^{-}(t)=\#\{j:M_j\le -t\}.
\]
Here \(R(t)\) is the total number of selected grouped features at threshold \(t\), while \(V(t)\) is the number of false discoveries. The realized false discovery proportion is
\[
\operatorname{FDP}(t)
=
\frac{V(t)}{R(t)\vee1}.
\]
The Group Gaussian Mirror procedure estimates the FDP by \(\widehat{\operatorname{FDP}}(t)
=
\frac{1+N^{-}(t)}{R(t)\vee1}\) and chooses
\[
\tau_q
=
\inf\left\{
t>0:
\widehat{\operatorname{FDP}}(t)\le q
\right\}.
\]
We first use the null symmetry property. For every null grouped feature \(j\in S_0\),
\[
M_j\overset{d}{=}-M_j.
\]
Consequently, for every \(t>0\),
\[
\mathbb P(M_j\ge t)
=
\mathbb P(M_j\le -t),
\quad
\mathbb E[V(t)]
=
\mathbb E[V^{-}(t)].
\]
Thus, at the population level, the positive null tail and the negative null tail contain the same expected number of grouped features. To transfer this expectation-level equality to empirical tail counts, we impose the standard aggregate weak-dependence condition used in Gaussian-Mirror-type and data-splitting FDR proofs. Let \(p_0=|S_0|\). Assume that, for some constants \(C>0\) and \(\alpha\in(0,2)\),
\[
\sup_{s\in\mathbb R}
\operatorname{Var}
\left(
\sum_{j\in S_0}\mathbf 1\{M_j>s\}
\right)
\le Cp_0^\alpha.
\]
This condition only requires that the dependence among the null mirror-tail indicators is not so dense that their aggregate variance has order \(p_0^2\). In particular, strong within-group dependence, and even strong dependence among some null groups, are allowed provided that the overall dependence is sufficiently sparse or weak in aggregate. We now show that this variance condition implies uniform concentration of the null tails. Define
\[
\widehat G_{+}(t)
=
\frac{V(t)}{p_0},
\qquad
G_{+}(t)
=
\frac{\mathbb E[V(t)]}{p_0}.
\]
For every fixed \(t\), Chebyshev's inequality gives
\[
\mathbb P
\left(
\left|
\widehat G_{+}(t)-G_{+}(t)
\right|>\varepsilon
\right)
\le
\frac{Cp_0^\alpha}
{\varepsilon^2p_0^2}
=
\frac{C}{\varepsilon^2}
p_0^{\alpha-2}
\longrightarrow0.
\]
Because both \(\widehat G_{+}(t)\) and \(G_{+}(t)\) are monotone, the usual finite-grid and monotonicity yields
\[
\sup_{t>0}
\left|
\frac{V(t)-\mathbb E[V(t)]}{p_0}
\right|
=o_p(1).
\quad
\sup_{t>0}
\left|
\frac{V^{-}(t)-\mathbb E[V^{-}(t)]}{p_0}
\right|
=o_p(1).
\]
Combining these two uniform concentration statements with
\(
\mathbb E[V(t)]
=
\mathbb E[V^{-}(t)]
\)
gives
\[
\sup_{t>0}
\frac{|V(t)-V^{-}(t)|}{p_0}
=o_p(1).
\]
Therefore, the empirical positive and negative null tails are uniformly asymptotically balanced. It remains to evaluate this relation at the data-dependent threshold \(\tau_q\). We impose the standard non-degeneracy condition that there exist constants \(T>0\) and \(c>0\) such that
\[
\mathbb P(\tau_q\le T)\longrightarrow1, \quad
\mathbb P\bigl(R(T)\ge cp_0\bigr)\longrightarrow1.
\]
This condition ensures that the procedure admits a feasible threshold and makes a non-vanishing number of discoveries. Since \(R(t)\) is non-increasing in \(t\), on the event \(\{\tau_q\le T\}\), we have 
\(
R(\tau_q)\ge R(T)
\). Hence, with probability tending to one,
\(
R(\tau_q)\ge cp_0
\). Using the uniform null-tail balance, we obtain
\[
\frac{
|V(\tau_q)-V^{-}(\tau_q)|
}{
R(\tau_q)\vee1
}
\le
\frac{
\sup_{0<t\le T}|V(t)-V^{-}(t)|
}{
cp_0
}
=o_p(1).
\]
It follows that
\[
\operatorname{FDP}(\tau_q)
=
\frac{V(\tau_q)}{R(\tau_q)\vee1}
\le
\frac{V^{-}(\tau_q)}
{R(\tau_q)\vee1}
+
o_p(1).
\]
Since every negative null statistic is also counted by the total negative tail,
\(
V^{-}(t)\le N^{-}(t)
\), for every \(t>0\). Therefore,
\[
\operatorname{FDP}(\tau_q)
\le
\frac{N^{-}(\tau_q)}
{R(\tau_q)\vee1}
+
o_p(1)
\le
\frac{1+N^{-}(\tau_q)}
{R(\tau_q)\vee1}
+
o_p(1).
\]
By the definition of \(\tau_q\),
\[
\frac{1+N^{-}(\tau_q)}
{R(\tau_q)\vee1}
\le q.
\]
\[
\operatorname{FDP}(\tau_q)
\le q+o_p(1).
\]
Finally, since \(0\le\operatorname{FDP}(\tau_q)\le1\), boundedness implies uniform integrability. Thus,
\[
\limsup_{p_0\to\infty}
\operatorname{FDR}
=
\limsup_{p_0\to\infty}
\mathbb E
\bigl[
\operatorname{FDP}(\tau_q)
\bigr]
\le q.
\]
Therefore, under null symmetry, the aggregate weak-dependence condition, and the non-degenerate-threshold condition, the Group Gaussian Mirror procedure asymptotically controls the false discovery rate at the target level \(q\). For the high-dimensional two-stage procedure, the same argument is applied conditionally on the first-stage screened set. On the sure-screening event, the second-stage statistics for retained null grouped features inherit the low-dimensional null symmetry, while unretained features have mirror statistic zero and do not enter either tail for \(t>0\). If the aggregate variance and non-degeneracy conditions hold conditionally for the second-stage mirror statistics and the sure-screening probability tends to one, then the same proof yields success asymptotic FDR control. This completes the proof.
\end{proof}

\subsection{Proof of Proposition \ref{prop:gls_whitened_ggm}}

\glswhitenedggm*

\begin{proof}
Consider the grouped linear model
\[
Y=\sum_{k=1}^p X_k\boldsymbol\beta_k+\varepsilon,
\qquad
\varepsilon\sim N(0,\sigma^2\Sigma),
\]
where \(\Sigma\succ 0\). Let \(L\) satisfy \(L^\top L=\Sigma^{-1}\), following the above definition, we have
\[
\bar Y
=
LY
=
\sum_{k=1}^p LX_k\boldsymbol\beta_k+L\varepsilon
=
\sum_{k=1}^p \bar X_k\boldsymbol\beta_k+\bar\varepsilon .
\]
Moreover, \(\mathbb E[\bar\varepsilon]=0\) and \(\operatorname{Cov}(\bar\varepsilon)=L\operatorname{Cov}(\varepsilon)L^\top=\sigma^2 L\Sigma L^\top\). Since \(L\) is non-singular and
\[
\Sigma=(L^\top L)^{-1}=L^{-1}L^{-\top}.
\]
Therefore,
\[
L\Sigma L^\top
=
L(L^{-1}L^{-\top})L^\top
=
I_n .
\]
Hence
\[
\bar\varepsilon\sim N(0,\sigma^2 I_n).
\]
Thus the whitened system is a grouped linear model with spherical Gaussian noise. We then show that OLS in the whitened system is equivalent to GLS in the original
system. For any coefficient vector \(\boldsymbol\beta\), where
\(\bar X=[\bar X_1,\ldots,\bar X_p]\) and \(X=[X_1,\ldots,X_p]\),
\[
\|\bar Y-\bar X\boldsymbol\beta\|_2^2
=
\|L(Y-X\boldsymbol\beta)\|_2^2 .
\]
Expanding the right-hand side gives
\[
\|L(Y-X\boldsymbol\beta)\|_2^2
=
(Y-X\boldsymbol\beta)^\top L^\top L(Y-X\boldsymbol\beta)
=
(Y-X\boldsymbol\beta)^\top\Sigma^{-1}(Y-X\boldsymbol\beta).
\]
This is exactly the GLS objective in the original system. Therefore, minimizing the OLS objective after whitening is equivalent to minimizing the GLS objective before whitening.

Now, it only remains to relate the mirror construction in the whitened system to the \(\Sigma^{-1}\)-inner product in the original system. For any two original-scale
vectors \(a,b\in\mathbb R^n\),
\[
\langle La,Lb\rangle
=
(La)^\top(Lb)
=
a^\top L^\top Lb
=
a^\top\Sigma^{-1}b .
\]
Thus the Euclidean inner product after whitening is exactly the
\(\Sigma^{-1}\)-inner product before whitening. In particular, the Euclidean projection onto the column space of
\(\bar X_{-j}=LX_{-j}\) is
\[
P_{\bar X_{-j}}
=
\bar X_{-j}
(\bar X_{-j}^\top \bar X_{-j})^{-1}
\bar X_{-j}^\top .
\]
Substituting \(\bar X_{-j}=LX_{-j}\), we get
\[
P_{\bar X_{-j}}
=
LX_{-j}
(X_{-j}^\top L^\top L X_{-j})^{-1}
X_{-j}^\top L^\top
=
LX_{-j}
(X_{-j}^\top\Sigma^{-1}X_{-j})^{-1}
X_{-j}^\top L^\top .
\]
Define the GLS projection in the original system by
\[
P_{X_{-j}}^\Sigma
=
X_{-j}
(X_{-j}^\top\Sigma^{-1}X_{-j})^{-1}
X_{-j}^\top\Sigma^{-1}.
\]
Since \(\Sigma^{-1}=L^\top L\), we have
\[
LP_{X_{-j}}^\Sigma
=
LX_{-j}
(X_{-j}^\top\Sigma^{-1}X_{-j})^{-1}
X_{-j}^\top\Sigma^{-1}
=
LX_{-j}
(X_{-j}^\top\Sigma^{-1}X_{-j})^{-1}
X_{-j}^\top L^\top L .
\]
On the other hand, we also have
\[
P_{\bar X_{-j}}L
=
LX_{-j}
(X_{-j}^\top\Sigma^{-1}X_{-j})^{-1}
X_{-j}^\top L^\top L .
\]
Therefore, \(P_{\bar X_{-j}}L=LP_{X_{-j}}^\Sigma\). Consequently, if \(M_{\bar X_{-j}}=I_n-P_{\bar X_{-j}}\) and \(M_{X_{-j}}^\Sigma=I_n-P_{X_{-j}}^\Sigma\), then \(M_{\bar X_{-j}}L=LM_{X_{-j}}^\Sigma\). The same correspondence applies to the residualized group block. In the whitened system,
\[
\bar R_j
=
M_{\bar X_{-j}}\bar X_j
=
M_{\bar X_{-j}}LX_j
=
LM_{X_{-j}}^\Sigma X_j .
\]
Thus, if \(R_j^\Sigma=M_{X_{-j}}^\Sigma X_j\), we have \(\bar R_j=LR_j^\Sigma\). The whitened Gram matrix is therefore
\[
\bar R_j^\top \bar R_j
=
(R_j^\Sigma)^\top L^\top L R_j^\Sigma
=
(R_j^\Sigma)^\top \Sigma^{-1}R_j^\Sigma .
\]
Hence the Euclidean Gram matrix in the whitened system is exactly the \(\Sigma^{-1}\)-metric Gram matrix in the original system.

Now consider the mirror noise construction. In the whitened system, the OLS Group
Gaussian Mirror projects the noise block away from the column spaces of
\(\bar X_{-j}\) and \(\bar R_j\), so that the projected noise block
\(\bar Q_j\) satisfies
\[
\bar Q_j^\top \bar X_{-j}=0,
\qquad
\bar Q_j^\top \bar R_j=0.
\]
Writing \(\bar Q_j=LQ_j\), \(\bar X_{-j}=LX_{-j}\), and
\(\bar R_j=LR_j^\Sigma\), these orthogonality relations become
\[
Q_j^\top L^\top L X_{-j}=0,
\qquad
Q_j^\top L^\top L R_j^\Sigma=0.
\]
Since \(L^\top L=\Sigma^{-1}\), this is equivalent to
\[
Q_j^\top\Sigma^{-1}X_{-j}=0,
\qquad
Q_j^\top\Sigma^{-1}R_j^\Sigma=0.
\]
Thus orthogonality after whitening is the same as \(\Sigma^{-1}\)-orthogonality in the original system. 

Similarly, the normalization in the whitened system, \(\widetilde{\bar Z}_j^\top \widetilde{\bar Z}_j=I_m\), corresponds to \(\widetilde Z_j^\top\Sigma^{-1}\widetilde Z_j=I_m\) in the original system, because \(\widetilde{\bar Z}_j=L\widetilde Z_j\)
implies
\[
\widetilde{\bar Z}_j^\top \widetilde{\bar Z}_j
=
\widetilde Z_j^\top L^\top L\widetilde Z_j
=
\widetilde Z_j^\top\Sigma^{-1}\widetilde Z_j .
\]
Therefore, the OLS Group Gaussian Mirror construction in the whitened system is identical to the GLS Group Gaussian Mirror construction in the original system with all projections, normalizations, and orthogonality relations computed under
the \(\Sigma^{-1}\)-inner product.

Finally, since the whitened error \(\bar\varepsilon\sim N(0,\sigma^2 I_n)\), the OLS Group Gaussian Mirror applies directly to the whitened system. Since OLS estimation in the whitened system is equivalent to GLS estimation in the original system, applying OLS Group Gaussian Mirror to \((\bar Y,\bar X,\bar\varepsilon)\) is equivalent to applying GLS Group Gaussian Mirror to \((Y,X,\varepsilon)\) with the \(\Sigma^{-1}\)-inner product. This completes the proof.
\end{proof}

\subsection{Proof of Proposition \ref{prop:post_selection_not_grouped}}

\postselectionnotgrouped*

\begin{proof}
The first-stage Group Lasso estimator is defined by
\[
\widehat{\boldsymbol\beta}^{\,GL}
\in
\arg\min_{\{\boldsymbol\beta_j\}_{j=1}^p}
\left\{
\frac{1}{2n}
\left\|
Y-\sum_{j=1}^p X_j\boldsymbol\beta_j
\right\|_2^2
+
\lambda\sum_{j=1}^p w_j\|\boldsymbol\beta_j\|_2
\right\}.
\]
Let \(\widehat{\mathcal A}=\{j:\widehat{\boldsymbol\beta}^{\,GL}_j\neq \boldsymbol 0\}\) be the selected grouped feature set, and we define the fitted residual \(\widehat r=Y-\sum_{j=1}^pX_j\widehat{\boldsymbol\beta}^{\,GL}_j\). Then the Karush-Kuhn-Tucker (KKT) conditions for Group Lasso imply that, for each active grouped feature \(j\in\widehat{\mathcal A}\),
\[
X_j^\top \widehat r
=
n\lambda w_j
\frac{\widehat{\boldsymbol\beta}^{\,GL}_j}
{\|\widehat{\boldsymbol\beta}^{\,GL}_j\|_2},
\]
whereas for each inactive grouped feature \(j\notin\widehat{\mathcal A}\),
\[
\|X_j^\top \widehat r\|_2
\le
n\lambda w_j .
\]
Therefore, the event \(\widehat{\mathcal A}=A\) is characterized by constraints involving group \(\ell_2\)-norms and, for active grouped features, normalized direction vectors of the form \(\widehat{\boldsymbol\beta}^{\,GL}_j/
\|\widehat{\boldsymbol\beta}^{\,GL}_j\|_2\). This differs from the coordinate-wise lasso. In the ordinary lasso, the inactive KKT condition is
\[
|x_j^\top \widehat r|\le n\lambda,
\]
which is equivalent to two linear inequalities. Together with a fixed active sign pattern, the lasso selection event can be represented as a finite intersection of affine half-spaces in \(Y\), and hence is polyhedral. For Group Lasso, however, the inactive grouped feature condition
\[
\|X_j^\top \widehat r\|_2\le n\lambda w_j
\]
is a second-order cone constraint whenever the group dimension is larger than one. Such a constraint is generally not representable as a finite intersection of linear inequalities. Moreover, the active grouped feature condition depends on the
continuous direction
\[
\widehat{\boldsymbol\beta}^{\,GL}_j/
\|\widehat{\boldsymbol\beta}^{\,GL}_j\|_2,
\]
rather than on a finite sign pattern as in the coordinate-wise lasso. Therefore, the event \(\widehat{\mathcal A}=A\) induced by Group Lasso is not polyhedral in the response vector. This completes the proof.

Also, this non-polyhedral structure can be seen explicitly in a simple example. Consider a model with a single group \(X_1\in\mathbb R^{n\times m}\), where
\(m>1\), and suppose \(X_1^\top X_1=nI_m\). In this case, the Group Lasso solution is obtained by group soft-thresholding,
and the group is selected if and only if
\[
\|X_1^\top Y\|_2>n\lambda w_1 .
\]
Equivalently, the non-selection event is
\[
\|X_1^\top Y\|_2\le n\lambda w_1 .
\]
This is the inverse image of an \(m\)-dimensional Euclidean ball under the
linear map \(Y\mapsto X_1^\top Y\). For \(m>1\), this set has a curved boundary and is not a polyhedron in \(Y\).
\end{proof}

\subsection{Proof of Theorem \ref{thm:two_stage_ggm}}

\twostageggm*

\begin{proof}
Fix a null grouped feature \(j\in\mathcal S_0\). Let
\(\widehat{\mathcal A}\) denote the first-stage selected grouped feature set. Since \(\widehat{\mathcal A}\) is computed only from the first data chunk \((X^{(1)},\boldsymbol y^{(1)})\), it is measurable with respect to the first-stage data. Recall that the extended two-stage mirror statistic is defined by
\[
\widetilde M_j
=
\begin{cases}
M_j, & j\in\widehat{\mathcal A},\\
0, & j\notin\widehat{\mathcal A},
\end{cases}
\]
where \(M_j\) is the Group Gaussian Mirror statistic computed on the second data chunk within the restricted model indexed by \(\widehat{\mathcal A}\). For any \(t>0\), by the law of total expectation,
\[
\mathbb P(\widetilde M_j>t)-\mathbb P(\widetilde M_j<-t)
=
\mathbb E\left[
\mathbb P(\widetilde M_j>t\mid \widehat{\mathcal A})
-
\mathbb P(\widetilde M_j<-t\mid \widehat{\mathcal A})
\right].
\]
We now analyze the conditional difference inside the expectation. Firstly, we consider the event \(j\notin\widehat{\mathcal A}\). By construction, \(\widetilde M_j=0\). Hence, for every \(t>0\),
\[
\mathbb P(\widetilde M_j>t\mid \widehat{\mathcal A})
=
\mathbb P(\widetilde M_j<-t\mid \widehat{\mathcal A})
=
0.
\]
Thus the conditional contribution is exactly zero on
\(\{j\notin\widehat{\mathcal A}\}\). Next, we consider the event \(j\in\widehat{\mathcal A}\). On this event, \(\widetilde M_j=M_j\), and \(M_j\) is obtained by applying the low-dimensional Group Gaussian Mirror to the second data chunk restricted to the selected grouped features
\(\widehat{\mathcal A}\). We define
\[
\mathcal E_n
=
\left\{
\mathcal S_1\subseteq\widehat{\mathcal A}
\right\}
\cap
\left\{
\text{the second-stage restricted design is low-dimensional and full-rank}
\right\}.
\]
By assumptions (ii) and (iii), we have
\[
\mathbb P(\mathcal E_n)\to1.
\]
On \(\mathcal E_n\), the screened model contains all nonzero grouped features. Therefore, for \(j\in\mathcal S_0\cap\widehat{\mathcal A}\), the coefficient block of \(j\) remains null in the second-stage restricted model, and no omitted-signal bias is introduced by the screening step. Moreover, by assumption (i), the second data chunk is independent of \(\widehat{\mathcal A}\), or at least asymptotically independent in the ordered-data case. Hence, conditional on \(\widehat{\mathcal A}\), the second-stage procedure is equivalent, up to an \(o(1)\) perturbation in the asymptotically independent case, to applying the low-dimensional Group Gaussian Mirror to a fixed deterministic restricted model. On \(\mathcal E_n\), this restricted model satisfies the low-dimensional full-rank conditions required by Theorem~\ref{thm:low_dim_ggm}. Therefore, we have
\[
\sup_{t>0}
\left|
\mathbb P(M_j>t\mid \widehat{\mathcal A})
-
\mathbb P(M_j<-t\mid \widehat{\mathcal A})
\right|
\to0 \quad \text{on} \quad \mathcal E_n.
\]
Since \(\widetilde M_j=M_j\) whenever
\(j\in\widehat{\mathcal A}\), it follows that
\[
\sup_{t>0}
\left|
\mathbb P(\widetilde M_j>t\mid \widehat{\mathcal A})
-
\mathbb P(\widetilde M_j<-t\mid \widehat{\mathcal A})
\right|
\to0 \quad \text{on} \quad \{j\in\widehat{\mathcal A}\}\cap\mathcal E_n.
\]
Combining the two cases, define
\[
Z_n
=
\sup_{t>0}
\left|
\mathbb P(\widetilde M_j>t\mid \widehat{\mathcal A})
-
\mathbb P(\widetilde M_j<-t\mid \widehat{\mathcal A})
\right|.
\]
Then \(0\le Z_n\le1\). On the event \(j\notin\widehat{\mathcal A}\), we have \(Z_n=0\). On the event \(\{j\in\widehat{\mathcal A}\}\cap\mathcal E_n\), the preceding conditional
low-dimensional argument gives \(Z_n\to0\). Since
\(\mathbb P(\mathcal E_n^c)\to0\), we conclude that
\[
Z_n\overset{p}{\to}0.
\]
Because \(Z_n\) is uniformly bounded, convergence in probability to zero implies \(\mathbb E[Z_n]\to0\). Finally, 
\[
\begin{aligned}
&
\sup_{t>0}
\left|
\mathbb P(\widetilde M_j>t)-\mathbb P(\widetilde M_j<-t)
\right| \\
&\qquad
=
\sup_{t>0}
\left|
\mathbb E\left[
\mathbb P(\widetilde M_j>t\mid \widehat{\mathcal A})
-
\mathbb P(\widetilde M_j<-t\mid \widehat{\mathcal A})
\right]
\right| \\
&\qquad
\le
\mathbb E\left[
\sup_{t>0}
\left|
\mathbb P(\widetilde M_j>t\mid \widehat{\mathcal A})
-
\mathbb P(\widetilde M_j<-t\mid \widehat{\mathcal A})
\right|
\right] \\
&\qquad
=
\mathbb E[Z_n]
\to0.
\end{aligned}
\]
This completes the proof.
\end{proof}

\subsection{Proof of Proposition \ref{prop:linear_shap_derivative}}

\linearshapderivative*

\begin{proof}
Consider the grouped linear model
\[
f(\boldsymbol x)
=
\alpha+\sum_{\ell=1}^p \boldsymbol x_\ell^\top\boldsymbol\beta_\ell,
\quad
\boldsymbol x_\ell\in\mathbb R^m.
\]
Fix a grouped feature \(j\), a permutation \(\pi\), and a background draw \(X'\). The two hybrid inputs in the marginal contribution \(\Delta_j(\boldsymbol x;\pi,X')\) differ only in the coordinates of group \(j\). All groups appearing before \(j\) in the permutation are taken from \(\boldsymbol x\) in both hybrid inputs, while all groups appearing after \(j\)
are taken from \(X'\) in both hybrid inputs. Therefore, by linearity, we have
\[
\begin{aligned}
\Delta_j(\boldsymbol x;\pi,X')
&=
f\bigl(\boldsymbol{x}_{P_j(\pi)\cup\{j\}},
X'_{-(P_j(\pi)\cup\{j\})}\bigr)
-
f\bigl(\boldsymbol{x}_{P_j(\pi)},X'_{-P_j(\pi)}\bigr)  \\
&=
\boldsymbol\beta_j^\top \boldsymbol x_j
-
\boldsymbol\beta_j^\top X'_j =
\boldsymbol\beta_j^\top(\boldsymbol x_j-X'_j).
\end{aligned}
\]
This expression does not depend on the permutation \(\pi\). Taking expectation over \(\Pi\) and \(X'\), we obtain
\[
\Phi_j(\boldsymbol x)
=
\mathbb E_{\Pi,X'}\!\left[
\Delta_j(\boldsymbol x;\Pi,X')
\right]
=
\boldsymbol\beta_j^\top
\left(
\boldsymbol x_j-\mathbb E[X'_j]
\right).
\]
\[
\nabla_{\boldsymbol x_j}\Phi_j(\boldsymbol x)
=
\nabla_{\boldsymbol x_j}
\left\{
\boldsymbol\beta_j^\top
\left(
\boldsymbol x_j-\mathbb E[X'_j]
\right)
\right\}
=
\boldsymbol\beta_j.
\]
By the definition of the block-level Permutation SHAP derivative vector, we have
\[
\boldsymbol{\psi}_j(\boldsymbol{x})
=
\left(
\frac{\partial \Phi_j(\boldsymbol{x})}{\partial x_{j,1}},
\ldots,
\frac{\partial \Phi_j(\boldsymbol{x})}{\partial x_{j,m}}
\right)^\top
=
\nabla_{\boldsymbol x_j}\Phi_j(\boldsymbol x)
=
\boldsymbol\beta_j.
\]
This completes the proof.
\end{proof}

\subsection{Proof of Theorem \ref{thm:shap_derivative_symmetry}}

\shapderivativesymmetry*

\begin{proof}
Fix a null grouped feature \(j\). We interpret the null at the model level: there exists a measurable function \(g\) such that
\[
f(\boldsymbol x)=g(\boldsymbol x_{-j}),
\qquad
\forall \boldsymbol x.
\]
Hence the fitted model does not depend on any sub-feature in the grouped block \(j\). We first show that the exact group-level Permutation SHAP value is identically zero. Fix any input \(\boldsymbol x\), any permutation \(\pi\), and any
background draw \(X'\). Since \(f\) does not depend on the grouped feature \(j\), the two hybrid inputs in the marginal contribution have the same value under \(f\):
\[
f\bigl(\boldsymbol{x}_{P_j(\pi)\cup\{j\}},
X'_{-(P_j(\pi)\cup\{j\})}\bigr)
=
f\bigl(\boldsymbol{x}_{P_j(\pi)},X'_{-P_j(\pi)}\bigr).
\]
Therefore, we have \(\Delta_j(\boldsymbol x;\pi,X')=0\) for every realization of \((\pi,X')\). Taking expectation over
\((\Pi,X')\) gives
\[
\Phi_j(\boldsymbol x)
=
\mathbb E_{\Pi,X'}\!\left[
\Delta_j(\boldsymbol x;\Pi,X')
\right]
=
0,
\qquad
\forall \boldsymbol x.
\]
Thus \(\Phi_j\) is the zero function, and every population partial derivative with respect to a sub-feature in group \(j\) is also zero:
\[
\frac{\partial \Phi_j(\boldsymbol x)}{\partial x_{j,k}}
=
0,
\qquad
k=1,\ldots,m.
\]
It now remains to show that the smoothed Monte Carlo estimator of this derivative is asymptotically symmetric. Fix a sub-feature \(x_{j,k}\) and a point \(x_0\). Let
\[
\widehat\Phi_{j,t}
=
\frac{1}{M_n}\sum_{q=1}^{M_n} Z_{tq},
\qquad
t=1,\ldots,n,
\]
denote the Monte Carlo Permutation SHAP estimator at the \(t\)-th explained point. Under the model-level null, the exact value satisfies \(\Phi_j(\boldsymbol x_t)=0\), so the Monte Carlo variables are centered:
\[
\mathbb E[Z_{tq}\mid \boldsymbol x_t]=0.
\]
Assume, as part of the regularity conditions, that conditional on the explained points \(\boldsymbol x_1,\ldots,\boldsymbol x_n\), the variables \(\{Z_{tq}\}_{1\le t\le n,1\le q\le M_n}\) are independent,
\[
\operatorname{Var}(Z_{tq}\mid \boldsymbol x_t)
=
\tau^2(x_{t,j,k}),
\]
and for some \(\delta>0\),
\[
\sup_{t,q}
\mathbb E\!\left[
|Z_{tq}|^{2+\delta}\mid \boldsymbol x_t
\right]
<\infty.
\]
We define \(U_t=\frac{x_{t,j,k}-x_0}{h_n}\), \(r_t=r(U_t)=(1,U_t,U_t^2)^\top\), and \(K_t=K(U_t)\). Then the local-quadratic coefficients \(\widehat b=(\widehat b_0,\widehat b_1,\widehat b_2)^\top\) solve
\[
\widehat b
=
\arg\min_{b\in\mathbb R^3}
\sum_{t=1}^n
K_t
\left(
\widehat\Phi_{j,t}-b^\top r_t
\right)^2.
\]
Let \(S_n=\sum_{t=1}^n K_t r_t r_t^\top\) and \(T_n=\sum_{t=1}^n K_t r_t \widehat\Phi_{j,t}\), the normal equations give \(\widehat b=S_n^{-1}T_n\). Therefore, the local-quadratic derivative estimator can be written as
\[
\widehat\psi_{j,k}(x_0)
=
\frac{\widehat b_1}{h_n}
=
\frac{1}{h_n}e_2^\top S_n^{-1}T_n,
\qquad
e_2=(0,1,0)^\top.
\]
Or equivalently,
\[
\widehat\psi_{j,k}(x_0)
=
\sum_{t=1}^n a_{nt}\widehat\Phi_{j,t},
\qquad
a_{nt}
=
\frac{1}{h_n}e_2^\top S_n^{-1}K_t r_t.
\]
Since
\[
\widehat\Phi_{j,t}
=
\frac{1}{M_n}\sum_{q=1}^{M_n}Z_{tq},
\]
we have the triangular-array representation
\[
\widehat\psi_{j,k}(x_0)
=
\sum_{t=1}^n\sum_{q=1}^{M_n}
\frac{a_{nt}}{M_n}Z_{tq}.
\]
Firstly, we compute the conditional mean. Because
\(\mathbb E[Z_{tq}\mid \boldsymbol x_t]=0\), we have
\[
\mathbb E[\widehat\Phi_{j,t}\mid \boldsymbol x_t]
=
\frac{1}{M_n}
\sum_{q=1}^{M_n}
\mathbb E[Z_{tq}\mid \boldsymbol x_t]
=
0.
\]
Hence
\[
\mathbb E[
\widehat\psi_{j,k}(x_0)
\mid
\boldsymbol x_1,\ldots,\boldsymbol x_n
]
=
\sum_{t=1}^n
a_{nt}
\mathbb E[\widehat\Phi_{j,t}\mid \boldsymbol x_t]
=
0.
\]
Next, we compute the conditional variance. Since the Monte Carlo draws are independent across both \(t\) and \(q\), we have \(\operatorname{Var}(\widehat\Phi_{j,t}\mid \boldsymbol x_t)=\tau^2(x_{t,j,k}) / M_n\). Therefore,
\[
\operatorname{Var}(T_n\mid \boldsymbol x_1,\ldots,\boldsymbol x_n)
=
\frac{1}{M_n}
\sum_{t=1}^n
K_t^2 r_t r_t^\top \tau^2(x_{t,j,k}).
\]
Using \(\widehat\psi_{j,k}(x_0)=h_n^{-1}e_2^\top S_n^{-1}T_n\),
we can obtain that
\[
\operatorname{Var}(
\widehat\psi_{j,k}(x_0)
\mid
\boldsymbol x_1,\ldots,\boldsymbol x_n
)
=
\frac{1}{h_n^2}
e_2^\top S_n^{-1}
\left[
\frac{1}{M_n}
\sum_{t=1}^n
K_t^2 r_t r_t^\top \tau^2(x_{t,j,k})
\right]
S_n^{-1}e_2.
\]
Define \(\Gamma_n=\frac{1}{nh_n}\sum_{t=1}^nK_t r_t r_t^\top\) and \(\Omega_n=\frac{1}{nh_n}\sum_{t=1}^nK_t^2 r_t r_t^\top \tau^2(x_{t,j,k})\), then we have
\[
S_n=nh_n\Gamma_n,
\qquad
\sum_{t=1}^n
K_t^2 r_t r_t^\top \tau^2(x_{t,j,k})
=
nh_n\Omega_n.
\]
Substituting these identities gives
\[
\operatorname{Var}(
\widehat\psi_{j,k}(x_0)
\mid
\boldsymbol x_1,\ldots,\boldsymbol x_n
)
=
\frac{1}{M_nnh_n^3}
e_2^\top
\Gamma_n^{-1}\Omega_n\Gamma_n^{-1}
e_2.
\]
Under the local design regularity conditions, \(\Gamma_n\to\Gamma\) and \(\Omega_n\to\Omega\) with \(\Gamma\) positive definite. Hence
\[
\operatorname{Var}(
\widehat\psi_{j,k}(x_0)
\mid
\boldsymbol x_1,\ldots,\boldsymbol x_n
)
=
\frac{1}{M_nnh_n^3}
V_{j,k}(x_0)
+
o\!\left(
\frac{1}{M_nnh_n^3}
\right),
\]
where
\[
V_{j,k}(x_0)
=
e_2^\top\Gamma^{-1}\Omega\Gamma^{-1}e_2.
\]
We now prove asymptotic normality. We define \(\xi_{ntq}=\frac{a_{nt}}{M_n}Z_{tq}\) and \(\widehat\psi_{j,k}(x_0)=\sum_{t=1}^n\sum_{q=1}^{M_n}\xi_{ntq}\). Conditional on the explained points, the variables \(\xi_{ntq}\) are
independent and centered. Since
\[
S_n^{-1}=(nh_n)^{-1}\Gamma_n^{-1},
\]
we have
\[
a_{nt}
=
\frac{1}{h_n}
e_2^\top
(nh_n)^{-1}\Gamma_n^{-1}K_t r_t
=
\frac{1}{nh_n^2}
e_2^\top\Gamma_n^{-1}K_t r_t.
\]
Because \(K\) is bounded and compactly supported, \(r_t\) is bounded on the support of \(K\). Since \(\Gamma_n^{-1}\) is bounded for all large \(n\), it should follow that
\[
\max_{1\le t\le n}|a_{nt}|
=
O\!\left(\frac{1}{nh_n^2}\right).
\]
We further let \(s_n^2=\operatorname{Var}(\widehat\psi_{j,k}(x_0)\mid\boldsymbol x_1,\ldots,\boldsymbol x_n)\). From the variance calculation above, we know that
\[
s_n^2
\asymp
\frac{1}{M_nnh_n^3}.
\]
To verify Lyapunov's condition, we use the moment bound. There exists \(C>0\) such that
\[
\sup_{t,q}
\mathbb E[
|Z_{tq}|^{2+\delta}
\mid
\boldsymbol x_t
]
\le C.
\]
Therefore,
\[
\begin{aligned}
\sum_{t=1}^n\sum_{q=1}^{M_n}
\mathbb E[
|\xi_{ntq}|^{2+\delta}
\mid
\boldsymbol x_1,\ldots,\boldsymbol x_n
]
&\le
C
\sum_{t=1}^n
M_n
\left(
\frac{|a_{nt}|}{M_n}
\right)^{2+\delta}.
\end{aligned}
\]
Only observations with \(U_t\) in the support of \(K\) contribute. By the local design regularity, the number of such observations is \(O(nh_n)\). Hence
\[
\sum_{t=1}^n\sum_{q=1}^{M_n}
\mathbb E[
|\xi_{ntq}|^{2+\delta}
\mid
\boldsymbol x_1,\ldots,\boldsymbol x_n
]
=
O\!\left(
nh_n\cdot M_n
\left(
\frac{1}{M_nnh_n^2}
\right)^{2+\delta}
\right)
=
O\!\left(
\frac{1}
{M_n^{1+\delta}n^{1+\delta}h_n^{3+2\delta}}
\right).
\]
On the other hand, we have
\[
s_n^{2+\delta}
\asymp
\left(
\frac{1}{M_nnh_n^3}
\right)^{1+\delta/2}
=
\frac{1}
{M_n^{1+\delta/2}n^{1+\delta/2}h_n^{3+3\delta/2}}.
\]
Since \(h_n\to0\), \(nh_n^3\to\infty\), and hence \(nh_n\to\infty\), we have
\[
\frac{
\sum_{t=1}^n\sum_{q=1}^{M_n}
\mathbb E[
|\xi_{ntq}|^{2+\delta}
\mid
\boldsymbol x_1,\ldots,\boldsymbol x_n
]
}
{s_n^{2+\delta}}
=
O\!\left(
\frac{1}{(M_nnh_n)^{\delta/2}}
\right)
\to 0,
\]
Thus, the Lyapunov's condition holds. Then, by Lyapunov's central limit theorem for triangular arrays, conditional on the explained points, we obtain
\[
\frac{\widehat\psi_{j,k}(x_0)}{s_n}
\overset{d}{\longrightarrow}
N(0,1),
\quad
\text{or equivalently},
\quad
\sqrt{M_nnh_n^3}\,
\widehat\psi_{j,k}(x_0)
\overset{d}{\longrightarrow}
N\!\left(0,V_{j,k}(x_0)\right).
\]
Till now, it only remains to translate asymptotic normality into null symmetry. Since \(s_n>0\), for any \(t>0\), we have
\[
\mathbb P(\widehat\psi_{j,k}(x_0)>t)
=
\mathbb P\left(
\frac{\widehat\psi_{j,k}(x_0)}{s_n}
>
\frac{t}{s_n}
\right)
\quad
\text{and}
\quad
\mathbb P(\widehat\psi_{j,k}(x_0)<-t)
=
\mathbb P\left(
\frac{\widehat\psi_{j,k}(x_0)}{s_n}
<
-\frac{t}{s_n}
\right).
\]
The standardized statistic converges in distribution to \(N(0,1)\), whose law is symmetric about zero. Since the limiting distribution function is continuous, the convergence of distribution functions should be uniform. Therefore, we have
\[
\sup_{u>0}
\left|
\mathbb P\left(
\frac{\widehat\psi_{j,k}(x_0)}{s_n}>u
\right)
-
\mathbb P\left(
\frac{\widehat\psi_{j,k}(x_0)}{s_n}<-u
\right)
\right|
\to0.
\]
Taking \(u=t/s_n\) yields
\[
\sup_{t>0}
\left|
\mathbb P\!\left(\widehat\psi_{j,k}(x_0)>t\right)
-
\mathbb P\!\left(\widehat\psi_{j,k}(x_0)<-t\right)
\right|
\to0.
\]
Since \(k\in\{1,\ldots,m\}\) was arbitrary, the same conclusion holds coordinate-wise for every sub-feature in the grouped feature block. This completes the proof.
\end{proof}

\subsection{Proof of Theorem \ref{thm:group_ngm_objective}}

\groupngmobjective*

\begin{proof}
Fix a grouped feature \(j\). For a given perturbation matrix
\(G_j\in\mathbb R^{m\times m}\), we define as before
\[
U_j(G_j)=X_j+\widetilde Z_jG_j,
\qquad
V_j(G_j)=X_j-\widetilde Z_jG_j,
\qquad
W_j=X_{-j}.
\]
For the \(i\)-th observation, we write
\[
t_i(G_j)
=
\bigl(u_i(G_j),v_i(G_j),w_i\bigr),
\]
where \(u_i(G_j)\), \(v_i(G_j)\), and \(w_i\) denote the \(i\)-th rows of \(U_j(G_j)\), \(V_j(G_j)\), and \(W_j\), respectively. The Neural Gaussian Mirror construction tests the conditional-independence target
\[
U_j(G_j)\perp V_j(G_j)\mid W_j.
\]
Let \(\eta(u,v,w)=\log p_{U,V,W}(u,v,w)\) be the log-density of the triplet \((U_j(G_j),V_j(G_j),W_j)\). As in the Neural Gaussian Mirror construction, we place \(\eta\) in the tensor-product RKHS:
\[
\mathcal H
=
\mathcal H_U\otimes\mathcal H_V\otimes\mathcal H_W.
\]
Decompose each marginal RKHS into its constant and centered components, \(\mathcal H_A=\mathcal H_A^0\oplus\mathcal H_A^1\) for \(A\in\{U,V,W\}\), we then have
\[
\mathcal H
=
\bigoplus_{a,b,c\in\{0,1\}}
\mathcal H_U^a\otimes\mathcal H_V^b\otimes\mathcal H_W^c.
\]
The conditional-independence null space should consist of components involving only main effects and interactions with \(W_j\), but no joint interaction between \(U_j\) and \(V_j\). Hence the alternative space is
\[
\mathcal H_1
=
(\mathcal H_U^1\otimes\mathcal H_V^1\otimes\mathcal H_W^0)
\oplus
(\mathcal H_U^1\otimes\mathcal H_V^1\otimes\mathcal H_W^1).
\]
Let \(K^1\) denote the reproducing kernel of this alternative space. Then the likelihood-ratio score in the alternative direction is the empirical RKHS score
\[
S_n(G_j)
=
\frac{1}{n}
\sum_{i=1}^n
K^1_{t_i(G_j)}.
\]
The grouped Neural Gaussian Mirror objective is defined as the squared RKHS norm
of this score:
\[
I_j^K(G_j)^2
=
\|S_n(G_j)\|_{\mathcal H}^2.
\]
By the reproducing property, we obtain
\[
\begin{aligned}
I_j^K(G_j)^2
&=
\left\langle
\frac{1}{n}\sum_{i=1}^nK^1_{t_i(G_j)},
\frac{1}{n}\sum_{\ell=1}^nK^1_{t_\ell(G_j)}
\right\rangle_{\mathcal H}  \\
&=
\frac{1}{n^2}
\sum_{i=1}^n\sum_{\ell=1}^n
K^1\!\left(t_i(G_j),t_\ell(G_j)\right).
\end{aligned}
\]
Because the alternative space requires centered main effects in the two mirror branches and allows arbitrary dependence on \(W_j\), the alternative kernel factorizes as
\[
K^1\!\left((u,v,w),(u',v',w')\right)
=
K_U^1(u,u')K_V^1(v,v')K_W(w,w').
\]
Since centering the \(U\)- and \(V\)-branch kernels is achieved by the centering matrix \(H_n=I_n-\frac{1}{n}\boldsymbol 1\boldsymbol 1^\top\), the Gram matrix of \(K_U^1\) and \(K_V^1\) are \(H_nK_U(G_j)H_n\) and \(H_nK_V(G_j)H_n\), respectively. The \(W_j\)-kernel matrix is denoted
by \(K_W\). Substituting the factorized kernel into the double sum gives
\[
\begin{aligned}
I_j^K(G_j)^2
&=
\frac{1}{n^2}
\sum_{i=1}^n\sum_{\ell=1}^n
\bigl[H_nK_U(G_j)H_n\bigr]_{i\ell}
\bigl[H_nK_V(G_j)H_n\bigr]_{i\ell}
\bigl[K_W\bigr]_{i\ell} \\
&=
\frac{1}{n^2}
\left[
(H_nK_U(G_j)H_n)
\circ
(H_nK_V(G_j)H_n)
\circ
K_W
\right]_{++}.
\end{aligned}
\]
This is exactly the stated grouped Neural Gaussian Mirror objective, thus completes the proof.
\end{proof}

\subsection{Proof of Proposition \ref{prop:linear_kernel_reduction}}

\linearkernelreduction*

\begin{proof}
We consider the linear-kernel special case and ignore the \(W_j\)-kernel weight, as stated in the proposition. Under linear kernels on the two mirror branches, we have
\[
K_U(G_j)=U_j(G_j)U_j(G_j)^\top,
\qquad
K_V(G_j)=V_j(G_j)V_j(G_j)^\top.
\]
Since \(H_n\) is symmetric and idempotent, define the centered mirror matrices \(U_{j,c}(G_j)=H_nU_j(G_j)\) and \(V_{j,c}(G_j)=H_nV_j(G_j)\), we then have
\[
H_nK_U(G_j)H_n
=
H_nU_j(G_j)U_j(G_j)^\top H_n
=
U_{j,c}(G_j)U_{j,c}(G_j)^\top,
\]
\[
H_nK_V(G_j)H_n
=
V_{j,c}(G_j)V_{j,c}(G_j)^\top.
\]
Ignoring the \(W_j\)-kernel weight, the grouped neural mirror objective is proportional to
\[
\left[
(H_nK_U(G_j)H_n)
\circ
(H_nK_V(G_j)H_n)
\right]_{++}.
\]
Substituting the two centered Gram matrices, we obtain
\[
\begin{aligned}
&
\left[
(H_nK_U(G_j)H_n)
\circ
(H_nK_V(G_j)H_n)
\right]_{++} \\
&\qquad =
\sum_{i=1}^n\sum_{\ell=1}^n
\left(
u_{j,c,i}(G_j)^\top u_{j,c,\ell}(G_j)
\right)
\left(
v_{j,c,i}(G_j)^\top v_{j,c,\ell}(G_j)
\right),
\end{aligned}
\]
where \(u_{j,c,i}(G_j)\) and \(v_{j,c,i}(G_j)\) denote the \(i\)-th rows of \(U_{j,c}(G_j)\) and \(V_{j,c}(G_j)\), respectively. The preceding double sum can be rewritten as a trace:
\[
\begin{aligned}
&
\sum_{i=1}^n\sum_{\ell=1}^n
\left(
u_{j,c,i}(G_j)^\top u_{j,c,\ell}(G_j)
\right)
\left(
v_{j,c,i}(G_j)^\top v_{j,c,\ell}(G_j)
\right)  \\
&\qquad =
\operatorname{tr}
\left(
U_{j,c}(G_j)U_{j,c}(G_j)^\top
V_{j,c}(G_j)V_{j,c}(G_j)^\top
\right) \\
&\qquad =
\operatorname{tr}
\left(
U_{j,c}(G_j)^\top V_{j,c}(G_j)
V_{j,c}(G_j)^\top U_{j,c}(G_j)
\right) \\
&\qquad =
\left\|
U_{j,c}(G_j)^\top V_{j,c}(G_j)
\right\|_F^2 .
\end{aligned}
\]
Therefore, we obtain
\[
I_j^K(G_j)^2
\propto
\frac{1}{n^2}
\left\|
U_{j,c}(G_j)^\top V_{j,c}(G_j)
\right\|_F^2
=
\left\|
\frac{1}{n}
\bigl(H_nU_j(G_j)\bigr)^\top
\bigl(H_nV_j(G_j)\bigr)
\right\|_F^2.
\]
The matrix \(\frac{1}{n}\bigl(H_nU_j(G_j)\bigr)^\top\bigl(H_nV_j(G_j)\bigr)\) is the centered empirical block cross-covariance between the two mirror
branches. Hence, in the linear-kernel case, minimizing \(I_j^K(G_j)^2\) is equivalent to driving this block cross-covariance toward zero, which is the block-orthogonality target in linear Group Gaussian Mirror. This completes the proof.
\end{proof}

\subsection{Proof of Theorem \ref{thm:ridge_symmetry_safe}}

\ridgesymmetrysafe*

\begin{proof}
Fix a grouped feature \(j\). Let \(Z_j\in\mathbb R^{n\times m}\) satisfy \(Z_j\overset d=-Z_j\) and \(Z_j\perp (X,\boldsymbol y)\) as described. By assumption, the projected noise block \(\widetilde Z_{j,\lambda}\) is
odd in \(Z_j\), and the learned perturbation matrix \(\widehat G_j\) is even in \(Z_j\), namely
\[
\widetilde Z_{j,\lambda}(-Z_j)
=
-\widetilde Z_{j,\lambda}(Z_j),
\qquad
\widehat G_j(-Z_j)=\widehat G_j(Z_j).
\]

The evenness of \(\widehat{G}_j\) is supported by the symmetry of the kernel objective. Since changing \(Z_j\) to \(-Z_j\) exchanges \(U_j\) and \(V_j\), and the objective is symmetric in the two branches, we have
\[
I_j^K(G;-Z_j)^2=I_j^K(G;Z_j)^2.
\]
Thus, the two optimization problems have identical argmin sets. The evenness of the selected \(\widehat{G}_j\) then follows if the minimizer is unique or a sign-invariant optimization procedure is used, with identical initialization and coupled optimization randomness.

The final mirror branches are constructed as \(U_j=X_j+\widetilde Z_{j,\lambda}\widehat G_j\) and \(V_j=X_j-\widetilde Z_{j,\lambda}\widehat G_j\). We first show that flipping the mirror noise swaps the two branches. Under the transformation \(Z_j\mapsto -Z_j\),
\[
\begin{aligned}
U_j(-Z_j)
&=
X_j+\widetilde Z_{j,\lambda}(-Z_j)\widehat G_j(-Z_j) \\
&=
X_j-\widetilde Z_{j,\lambda}(Z_j)\widehat G_j(Z_j) \\
&=
V_j(Z_j).
\end{aligned}
\]
\[
\begin{aligned}
V_j(-Z_j)
&=
X_j-\widetilde Z_{j,\lambda}(-Z_j)\widehat G_j(-Z_j) \\
&=
X_j+\widetilde Z_{j,\lambda}(Z_j)\widehat G_j(Z_j) \\
&=
U_j(Z_j).
\end{aligned}
\]
Thus the sign flip of the auxiliary mirror noise does not change the construction except for exchanging the plus and minus mirror branches. 
Let \(\Psi(y,W_j,U_j,V_j)=(O_j^+,O_j^-)\) denote the downstream training-and-attribution map. The assumed swap-equivariance requires that exchanging the two mirror branches, together with their corresponding initialization, optimization randomness, and SHAP sampling randomness, exchanges the resulting importance scores. Identical preprocessing, architecture, and optimization rules for both branches, together with an exchangeable initialization distribution, support this property in distribution. Under the assumed equivariance, we have
\[
\Psi(y,W_j,V_j,U_j)=(O_j^-,O_j^+),
\]
with the equality interpreted under consistently coupled randomness.

Combining this with the branch-swapping identity above gives
\[
(O_j^+,O_j^-)(-Z_j)
=
(O_j^-,O_j^+)(Z_j).
\]
Because \(Z_j\overset d=-Z_j\), the joint distribution of the full construction under \(Z_j\) is the same as under \(-Z_j\), except for the induced exchange of the two mirror branches. The construction of \(\widetilde Z_{j,\lambda}\) and \(\widehat G_j\) is label-free, and \(Z_j\) is independent of \((X,\boldsymbol y)\). Therefore, under the null,
\[
(O_j^+,O_j^-)
\overset d=
(O_j^-,O_j^+).
\]
This completes the proof.
\end{proof}

\paragraph{Remark (Null symmetry of mirror statistics).}
The branch exchangeability in Theorem 5 does not directly imply the null symmetry of the signed-inner-product mirror statistic, since \(M_j(O_j^+,O_j^-)=M_j(O_j^-,O_j^+)\). An additional sufficient condition is one-branch sign symmetry:
\[
(O_j^+,O_j^-)\overset{d}{=}(O_j^+,-O_j^-),
\]
which implies \(M_j\overset{d}{=}-M_j\). Asymptotically, this condition can follow when systematic null attributions are negligible relative to centered estimation errors and the errors satisfy joint one-branch sign symmetry, for example under a centered joint Gaussian limit with vanishing cross-covariance. The additional challenge of controlling systematic null attribution is further discussed in Appendix~\ref{multiplicity}.

\section{Multiplicity Concern and Practical Implications}
\label{multiplicity}

An important concern for our framework is predictive and explanation multiplicity: multiple models with similar predictive risk may rely on different feature relationships and yield different attributions. This is related to underspecification ~\cite{3586589.3586815} and attribution inconsistency within the Rashomon set ~\cite{laberge2023partial}. In our setting, some near-optimal networks may recover the intended relationship, while others may achieve similar prediction accuracy through correlated null features or feature combinations, corresponding to shortcut learning ~\cite{Geirhos2020}.

To formalize this issue, define the Rashomon set
\[
\mathcal{R}_n(\eta_n)
=
\left\{
f \in \mathcal{F}_n :
R(f)
\leq
\inf_{g\in\mathcal{F}_n}R(g)+\eta_n
\right\},
\]
and let
\[
\mu_j^\pm(f)
=
\mathbb{E}_{Q_j}
\left[
\psi_j^\pm(\mathbf{X};f)
\right]
\]
denote the population aggregated Permutation-SHAP derivative for the two mirror branches. A stronger model-class-level condition is null-attribution identifiability:
\begin{equation}
\sup_{f\in\mathcal{R}_n(\eta_n)}
\max_{j\in S_0}
\left\{
\left\|\mu_j^+(f)\right\|_2
+
\left\|\mu_j^-(f)\right\|_2
\right\}
\longrightarrow 0.
\tag{R1}
\end{equation}

Condition (R1) requires all sufficiently predictive models, rather than a single fitted model, to assign negligible attribution to null features. This is stronger than predictive consistency, since similar prediction error does not imply agreement on feature usage.

However, (R1) is stronger than necessary for one realized PSGM model. For the fitted predictor \(\widehat f_n\), the direct requirement is asymptotic null-attribution consistency. Writing
\[
\widehat{\psi}_j^\pm
=
\underbrace{\mu_j^\pm(\widehat f_n)}_{\text{systematic fitted-model attribution}}
+
\underbrace{r_n\varepsilon_j^\pm}_{\text{centered statistical error}}
+
o_p(r_n),
\]
we require
\begin{equation}
\max_{j\in S_0}
\frac{
\left\|\mu_j^+(\widehat f_n)\right\|_2
+
\left\|\mu_j^-(\widehat f_n)\right\|_2
}{r_n}
\xrightarrow{p}
0.
\tag{R2}
\end{equation}

Thus, finite-sample or optimization-induced null dependencies are allowed, but they must vanish relative to attribution-estimation noise. Persistent directional attribution to a null group represents systematic shortcut learning and can invalidate null mirror symmetry.

This distinction clarifies the scope of PSGM's FDR guarantee. The attribution layer is model-conditional: after fitting a predictor, Permutation SHAP characterizes how that predictor uses its inputs. If a fitted network relies on a feature satisfying
\[
Y \perp X_j \mid X_{-j},
\]
the feature remains statistically null but may be important for the fitted model. Its attribution can therefore be faithful to the model while misleading about the underlying data-generating mechanism. Condition (R2) connects these two notions.

More broadly, FDR procedures control errors relative to a specified statistical null under assumptions such as exchangeability or symmetry. They do not generally resolve which model in a Rashomon set represents the intended mechanism. Therefore, PSGM does not eliminate predictive multiplicity; its neural guarantee is conditional on sufficiently small systematic null attribution of the fitted model.

A possible extension is to define a model-class-level attribution target. ~\cite{fisher} proposed model class reliance, which evaluates variable importance over well-performing models, while ~\cite{laberge2023partial} study attribution consensus within the Rashomon set. This would directly address predictive multiplicity but requires a different inferential framework and new FDR analysis.

\paragraph{Some practical implications for model training.} Although generic training cannot guarantee null-attribution identifiability, the preceding analysis provides several practical principles. First, two mirror branches should receive identical preprocessing, architecture, regularization, initialization distributions, and optimization rules, so that the training procedure does not introduce an artificial plus/minus preference. Second, weight decay, dropout, early stopping, and restricted network capacity can reduce the ability of an overparameterized model to retain weak shortcut dependencies, although they do not guarantee that the selected model is the correct member of the Rashomon set. Third, models with comparable validation error can be compared across random initializations, with preference given to solutions whose attributions are more stable under retraining and mirror-branch relabeling.

When prior information about invalid dependencies is available, a more direct possibility is attribution regularization. For example, ~\cite{ross2017right} penalize model gradients on features that should not determine the prediction, thereby training models to be ``right for the right reasons.'' When multiple environments or predefined subgroups are available, invariant or group-robust objectives can similarly discourage relationships that are predictive only in one environment or subgroup. These techniques are possible mitigations of shortcut learning, rather than replacements for the explicit assumptions in our FDR theorem.

\section{Sensitivity Analysis for Monte Carlo Permutation SHAP}
\label{mcexperiment}

Here we conduct experiments to support the implementation choice in Section~\ref{pshap_derivative}. In PSGM, the exact group-level Permutation SHAP value is approximated by Monte Carlo permutation sampling before estimating the smoothed SHAP derivative. We therefore vary the number of permutation draws from 3 to 316 and evaluate whether the resulting FDR and power are sensitive to this approximation. The results show that increasing the number of draws has only a limited effect on selection performance, and that 10 permutation draws are sufficient for stable results in our simulations. We use this value in the main neural experiments.

\begin{table}[H]
\centering
\footnotesize
\setlength{\tabcolsep}{3.5pt}
\renewcommand{\arraystretch}{0.9}
\caption{FDR and power of PSGM under different numbers of Monte Carlo permutation draws}
\label{tab:neural_mc_simulation}
\begin{tabular}{@{}ccccccccccccc@{}}
\midrule
\multirow{2}{*}{\textbf{\makecell{Setting \\ (p*m, n)}}} &
\multirow{2}{*}{\textbf{Model}} &
\multirow{2}{*}{\textbf{\makecell{Signal \\ Strength}}} &
\multicolumn{2}{c}{\textbf{MC3}} &
\multicolumn{2}{c}{\textbf{MC10}} &
\multicolumn{2}{c}{\textbf{MC32}} &
\multicolumn{2}{c}{\textbf{MC100}} &
\multicolumn{2}{c}{\textbf{MC316}} \\
\cmidrule(lr){4-5}
\cmidrule(lr){6-7}
\cmidrule(lr){8-9}
\cmidrule(lr){10-11}
\cmidrule(lr){12-13}
& & & FDR & Power & FDR & Power & FDR & Power & FDR & Power & FDR & Power \\
\midrule

\multirow{6}{*}{$\makecell{p*m = 1500 \\ n = 2500}$}
& \multirow{2}{*}{LSTM}
& Weak
& 0.182 & \textbf{0.896}
& 0.161 & \textcolor{blue}{0.879}
& 0.143 & 0.844
& \textbf{0.123} & 0.812
& \textcolor{blue}{0.127} & 0.851 \\

&
& Strong
& \textcolor{blue}{0.148} & 0.916
& \textbf{0.145} & \textbf{0.958}
& 0.164 & 0.937
& 0.167 & \textcolor{blue}{0.943}
& 0.174 & 0.957 \\

& \multirow{2}{*}{GRU}
& Weak
& 0.132 & \textbf{0.934}
& 0.133 & \textcolor{blue}{0.931}
& 0.099 & 0.897
& \textbf{0.086} & 0.874
& \textcolor{blue}{0.097} & 0.904 \\

&
& Strong
& \textcolor{blue}{0.113} & 0.951
& 0.162 & 0.979
& \textbf{0.102} & \textcolor{blue}{0.985}
& 0.185 & 0.982
& 0.162 & \textbf{0.994} \\

& \multirow{2}{*}{Transformer}
& Weak
& 0.073 & \textbf{0.855}
& 0.068 & \textcolor{blue}{0.846}
& \textcolor{blue}{0.052} & 0.812
& \textbf{0.029} & 0.811
& 0.056 & 0.813 \\

&
& Strong
& 0.106 & 0.855
& 0.116 & 0.887
& \textcolor{blue}{0.083} & 0.893
& \textbf{0.082} & \textbf{0.895}
& 0.128 & \textcolor{blue}{0.893} \\

\midrule

\multirow{6}{*}{$\makecell{p*m = 3500 \\ n = 2500}$}
& \multirow{2}{*}{LSTM}
& Weak
& \textbf{0.111} & 0.615
& \textcolor{blue}{0.116} & \textbf{0.625}
& 0.119 & \textcolor{blue}{0.618}
& 0.149 & 0.617
& 0.168 & 0.616 \\

&
& Strong
& 0.152 & 0.865
& 0.152 & 0.886
& \textcolor{blue}{0.145} & \textcolor{blue}{0.895}
& 0.147 & \textbf{0.899}
& \textbf{0.144} & 0.893 \\

& \multirow{2}{*}{GRU}
& Weak
& 0.063 & 0.798
& \textcolor{blue}{0.062} & \textcolor{blue}{0.817}
& \textbf{0.041} & \textbf{0.820}
& 0.079 & 0.806
& 0.089 & 0.812 \\

&
& Strong
& 0.235 & 0.884
& \textbf{0.223} & 0.919
& \textcolor{blue}{0.224} & 0.911
& 0.243 & \textbf{0.931}
& 0.251 & \textcolor{blue}{0.926} \\

& \multirow{2}{*}{Transformer}
& Weak
& \textbf{0.104} & 0.492
& 0.141 & 0.491
& \textcolor{blue}{0.112} & \textcolor{blue}{0.503}
& 0.155 & \textbf{0.510}
& 0.128 & 0.501 \\

&
& Strong
& 0.237 & 0.874
& 0.283 & \textcolor{blue}{0.900}
& 0.254 & 0.889
& \textcolor{blue}{0.208} & \textbf{0.921}
& \textbf{0.196} & 0.892 \\

\bottomrule
\end{tabular}
\end{table}

\section{Convergence Analysis for Kernel-Based Objective Optimization}
\label{convergenceexperiment}

Here we conduct experiments to support the perturbation-matrix optimization step introduced in Section~\ref{ngm_background} and Section~\ref{neuralnetwork}. In the grouped neural mirror construction, the scalar perturbation in Neural Gaussian Mirror is replaced by a matrix \(G_j\), which is optimized by minimizing the kernel-based conditional-dependence objective \(I_j^K(G_j)^2\). We track the standardized objective value during optimization to verify that this step is numerically stable and does not require heavy tuning. The results show that initializing \(G_j=I_m\) and using a fixed learning rate and number of epochs yields stable convergence, motivating the fixed choice used in the later neural experiments.

\begin{figure}[H]
    \centering
    \includegraphics[width=0.9\textwidth]{"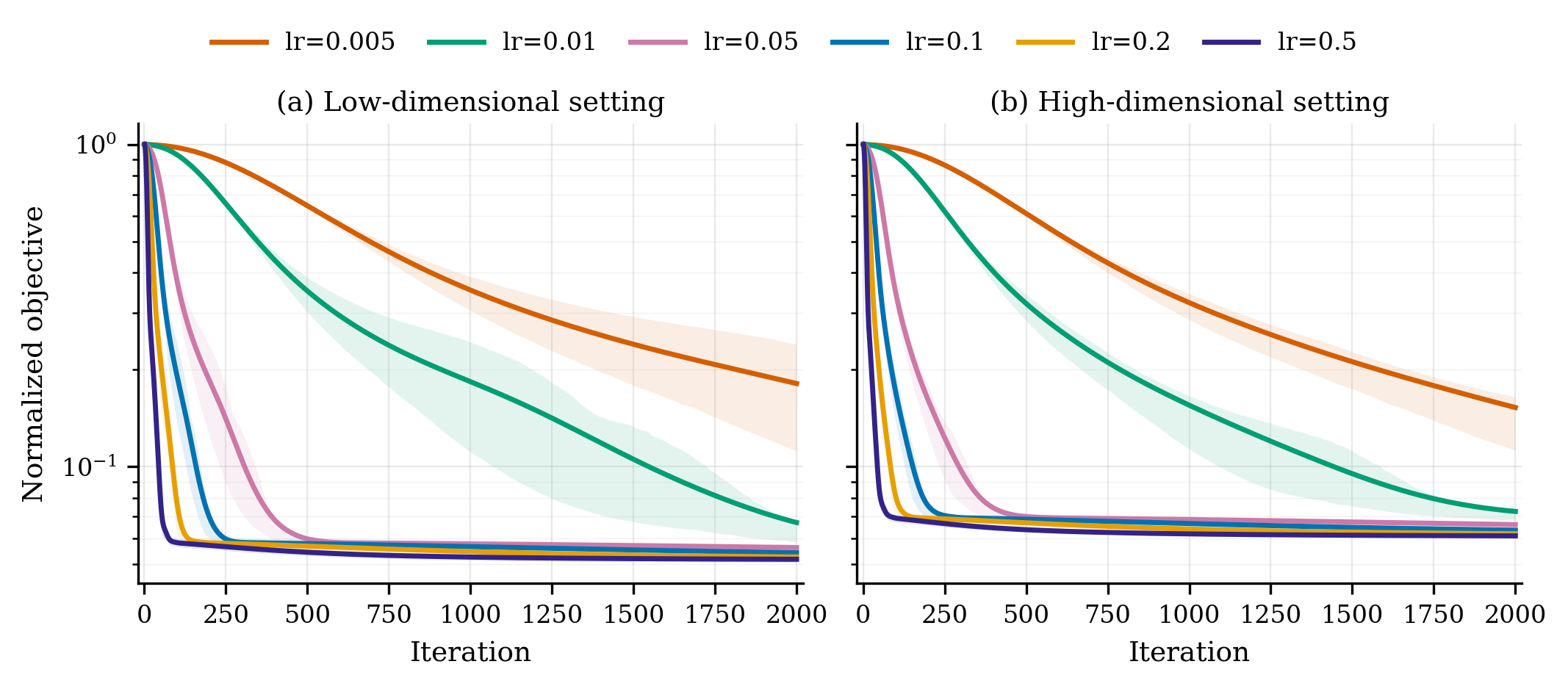"}
    \caption{Convergence curve of the kernel-based objective optimization}
    \label{gj_lr_neurips_combined}
\end{figure}

\section{Complexity and Scale Sensitivity Analysis for PSGM}
\label{scexperiment}

Here we show runtime/memory complexity and performance sensitivity analysis of PSGM to different scales of \(p, n, m\) in both low- and high-dimensional settings. 
First we scale \( ((p*m), n) \) by (2*2) for both cases and record average performance and runtime/memory all on same the hardware we used in the Numerical Simulation section. Empirically we find the performance degrades high-dim case with larger p and n. However it is more like a model predictive multiplicity issue introduced in see Appendix ~\ref{multiplicity}. We also find the whole algorithm is 6-8 times slower when (p*m*n) is 4 times larger. Also as we mentioned, since the $G_j$ optimization step is the slowest, this can effectively accelerated by increasing the learning rate.

Second, we add sensitivity analysis to scales of $m$ only. We fix $p$ and $n$ and test $m=3$ and $m=15$. Empirically as we stated in the paper, the performance degrades larger $m$. This is also mostly due to model predictive multiplicity issue. We also compare runtime and memory with the two baselines, demonstrating that our method achieves superior performance without adding too much algorithm complexity.

\begin{table}[H]
\footnotesize
\centering
\caption{Complexity and scale sensitivity analysis for PSGM}
\label{tab:scale_runtime}
\resizebox{\textwidth}{!}{
\begin{tabular}{@{}ccccccccccc@{}}
\toprule
\textbf{Experiment} & \textbf{Method}
& $p$ & $m$ & $n$
& \textbf{FDR} & \textbf{Power}
& \textbf{Runtime (min)}
& \textbf{Peak memory (GB)} \\
\midrule

Scale--LD & PSGM
& 300 & 5 & 2500
& 0.063
& 1.000
& 3.661
& 5.854 \\

Scale--LD & PSGM
& 600 & 5 & 5000
& 0.151
& 1.000
& 28.815
& 11.746 \\

Scale--LD & PSGM
& 1200 & 5 & 10000
& 0.218
& 1.000
& 153.544
& 24.093 \\

Scale--HD & PSGM
& 700 & 5 & 2500
& 0.091
& 0.667
& 5.972
& 13.869 \\

Scale--HD & PSGM
& 1400 & 5 & 5000
& 0.167
& 1.000
& 42.097
& 27.378 \\

Scale--HD & PSGM
& 2800 & 5 & 10000
& 0.583
& 1.000
& 340.836
& 56.162 \\

$m$ sensitivity & PSGM
& 300 & 3 & 2500
& 0.000
& 1.000
& 3.452
& 3.561 \\

$m$ sensitivity & PSGM
& 300 & 5 & 2500
& 0.063
& 1.000
& 3.661
& 5.854 \\

$m$ sensitivity & PSGM
& 300 & 15 & 2500
& 0.417
& 0.467
& 3.912
& 17.173 \\

\bottomrule
\end{tabular}}
\end{table}

\section{Additional Linear Simulation Results}
\label{additionalsimulation}

This appendix reports the performance and average runtime of another baseline, Group Knockoff \cite{Dai2016}, under the same experiment setting in Section ~\ref{linear_simulation}. We only experiment Group Knockoff on limited within-group AR coefficient and among-group correlation settings because the method is significantly slower. We found our method much faster yet performs similar.

This appendix also reports additional linear simulation results under a block-correlated Gaussian grouped design. This design complements the AR grouped design in Section~\ref{linear_simulation}: instead of generating each grouped feature from lagged values of a latent AR process, we generate a static Gaussian design with controlled within-group and between-group correlations. This allows us to check whether the performance of GGM depends on the temporal AR construction used in the main linear simulations.

In this design, each row of \(X\) is sampled from a multivariate Gaussian distribution. For each grouped feature \(X_j\), the \(m\) sub-features have within-group correlation \(\rho_w\in\{0.1,0.5,0.9\}\). Columns from different grouped features have among-group correlation \(\rho_b=\alpha\rho_w\), with \(\alpha\in\{0.2,0.8\}\). For each non-null grouped feature, the sub-feature weights are sampled independently from \(N(0,\sigma_\beta^2)\), rather than following the decreasing lag profile used in the AR grouped design. Table~\ref{tab:block_correlated_gaussian_linear} reports the averaged grouped-feature FDR and power over 50 repetitions. The results show that GGM maintains stable FDR control and high power under this alternative block-correlated Gaussian design.

\begin{table}[H]
\centering
\footnotesize
\setlength{\tabcolsep}{5pt}
\renewcommand{\arraystretch}{0.9}
\caption{Statistical performance and computational cost of Group Knockoff}
\label{tab:group_knockoff}
\begin{tabular}{@{}ccccccccccc@{}}
\toprule
\textbf{Within-group AR}
& \textbf{Among-Group corr.}
& $p$ & $m$ & $n$
& \textbf{FDR} & \textbf{Power} 
& \textbf{Runtime (min)}
\\
\midrule

0.1
& 0.2
& 300
& 5
& 2500
& 0.020
& 1.000
& 294.564 \\

0.1
& 0.8
& 300
& 5
& 2500
& 0.103
& 0.995
& 344.014 \\

0.5
& 0.2
& 300
& 5
& 2500
& 0.073
& 1.000
& 320.738 \\

0.5
& 0.8
& 300
& 5
& 2500
& 0.045
& 1.000
& 342.865 \\

\bottomrule
\end{tabular}
\end{table}

\begin{table}[H]
\centering
\footnotesize
\setlength{\tabcolsep}{5pt}
\renewcommand{\arraystretch}{0.9}
\caption{Linear simulation results under the block-correlated Gaussian grouped design}
\label{tab:block_correlated_gaussian_linear}
\begin{tabular}{@{}ccccccccccc@{}}
\midrule
\multirow{2}{*}{\textbf{\makecell{Setting \\ (p*m, n)}}} &
\multirow{2}{*}{\textbf{\makecell{Among-Group \\ Corr}}} &
\multirow{2}{*}{\textbf{\makecell{In-Group \\ Corr}}} &
\multicolumn{2}{c}{\textbf{GGM-OLS}} &
\multicolumn{2}{c}{\textbf{GGM-GLS}} &
\multicolumn{2}{c}{\textbf{Partial-F BH}} &
\multicolumn{2}{c}{\textbf{Group Lasso}} \\
\cmidrule(lr){4-5}
\cmidrule(lr){6-7}
\cmidrule(lr){8-9}
\cmidrule(lr){10-11}
& & & FDR & Power & FDR & Power & FDR & Power & FDR & Power \\
\midrule

\multirow{6}{*}{$\makecell{p*m = 1500 \\ n = 2500}$}
& \multirow{3}{*}{0.2}
& 0.1 & \textbf{0.076} & 1.000 & \textcolor{blue}{0.077} & 1.000 & 0.090 & 1.000 & 0.316 & 1.000 \\
&
& 0.5 & \textcolor{blue}{0.088} & 1.000 & 0.091 & 1.000 & \textbf{0.086} & 1.000 & 0.402 & 1.000 \\
&
& 0.9 & 0.104 & 0.997 & \textbf{0.078} & 0.997 & \textcolor{blue}{0.087} & 0.999 & 0.453 & 0.993 \\
&
\multirow{3}{*}{0.8}
& 0.1 & \textcolor{blue}{0.075} & 1.000 & 0.080 & 1.000 & \textbf{0.061} & 1.000 & 0.404 & 1.000 \\
&
& 0.5 & 0.111 & 1.000 & 0.124 & 1.000 & \textbf{0.076} & 1.000 & 0.381 & 1.000 \\
&
& 0.9 & \textbf{0.084} & 0.999 & \textcolor{blue}{0.085} & 0.999 & \textbf{0.084} & 0.999 & 0.395 & 1.000 \\

\midrule

\multirow{6}{*}{$\makecell{p*m = 3500 \\ n = 2500}$}
& \multirow{3}{*}{0.2}
& 0.1 & 0.085 & 1.000 & \textcolor{blue}{0.069} & 1.000 & \textbf{0.065} & 1.000 & 0.318 & 1.000 \\
&
& 0.5 & \textbf{0.083} & 1.000 & 0.097 & 1.000 & \textcolor{blue}{0.094} & 1.000 & 0.442 & 1.000 \\
&
& 0.9 & 0.111 & 1.000 & \textbf{0.065} & 1.000 & \textcolor{blue}{0.095} & 0.994 & 0.378 & 1.000 \\
&
\multirow{3}{*}{0.8}
& 0.1 & \textcolor{blue}{0.067} & 1.000 & 0.074 & 0.995 & \textbf{0.063} & 1.000 & 0.415 & 0.997 \\
&
& 0.5 & \textcolor{blue}{0.104} & 0.998 & 0.142 & 1.000 & \textbf{0.058} & 0.997 & 0.397 & 1.000 \\
&
& 0.9 & \textbf{0.086} & 0.996 & 0.091 & 1.000 & \textcolor{blue}{0.088} & 1.000 & 0.383 & 0.999 \\

\bottomrule
\end{tabular}
\end{table}

\section{Additional Details for the CausalRivers Real Data Application}
\label{causalriverdetails}

\subsection{Dataset and Preprocessing}

CausalRivers~\cite{Stein2025} is a large-scale real-world benchmark for causal discovery from river-discharge time series. In this paper, we use the East Germany subset, which contains 666 gauge stations observed at 15-minute resolution from January 2019 to December 2023, together with a directed upstream-to-downstream graph with 666 nodes and 651 edges. The dataset is useful for our setting because it combines realistic temporal dependence with physically meaningful graph information. However, the graph should only be treated as proxy relevance information in our regression problem. It should not be interpreted as a perfect statistical ground truth for our regression problem, since river discharge may also be affected by nonlinear flow, seasonality, weather-driven common shocks, variable causal lags, and unobserved confounding.

Let \(R_{t,k}\) denote the discharge measurement at station \(k\). We aggregate the raw 15-minute measurements to 3-hour averages, interpolate missing values, and use first differences \(\Delta R_{t,k}=R_{t,k}-R_{t-1,k}\) to reduce persistent level effects and common low-frequency trends. We use the first 2400 differenced observations after January 1, 2021, and split them chronologically with a train-test ratio of \(4:1\). For a target station \(j\), the real response is \(y_t^{\mathrm{real}}=\Delta R_{t,j}\). For each candidate station \(k\), we construct one grouped lag feature \(X_{t,k}=(\Delta R_{t-1,k},\Delta R_{t-2,k},\Delta R_{t-3,k})\), so the selected unit is a station-level lag block rather than an individual lagged covariate. All covariates and responses are standardized using the training split.

\subsection{Target and Grouped Feature Construction}

We select six target stations using a fixed \(X\)-only screening rule. A target is retained if it has at least 10 graph-relevant candidate groups, 50 eligible graph-null groups, a low-dimensional feasible design for the linear benchmark, and an acceptable condition number. The selected target station IDs are \(176,166,305,304,730,\) and \(167\). As shown in Table~\ref{tab:causalrivers_design_summary}, each target design contains 64--66 station groups, including 14--16 graph-relevant groups and 50 sampled graph-null groups. Since each station group has three lags, this gives 192--198 scalar lagged covariates.

\begin{table}[H]
\centering
\footnotesize
\setlength{\tabcolsep}{5pt}
\renewcommand{\arraystretch}{0.9}
\caption{Summary of the six CausalRivers target designs}
\label{tab:causalrivers_design_summary}
\small
\begin{tabular}{@{}cccccccccc@{}}
\toprule
Target ID & River system & \# Direct & \# 2nd & \# Null & \# Groups & \# Lagged covariates & Cond. no. \\
\midrule
176 & Elbe  & 10 & 6  & 50 & 66 & 198 & 49.284 \\
166 & Elbe  & 5  & 11 & 50 & 66 & 198 & 64.500 \\
305 & Havel & 6  & 9  & 50 & 65 & 195 & 24.152 \\
304 & Havel & 4  & 10 & 50 & 64 & 192 & 24.750 \\
730 & Saale & 6  & 8  & 50 & 64 & 192 & 25.281 \\
167 & Elbe  & 4  & 10 & 50 & 64 & 192 & 25.753 \\
\bottomrule
\end{tabular}
\end{table}

We define graph-proxy relevance using the directed river graph. Direct upstream stations are labeled as first-order relevant groups, and upstream stations of those direct upstream stations are labeled as second-order relevant groups. Graph-null groups are sampled from stations with no directed path relation to the target, excluding the target itself, its ancestors, and its descendants. Candidate stations with near-zero differenced variance are removed before fitting.

\begin{figure}[H]
    \centering
    \includegraphics[width=1\textwidth]{"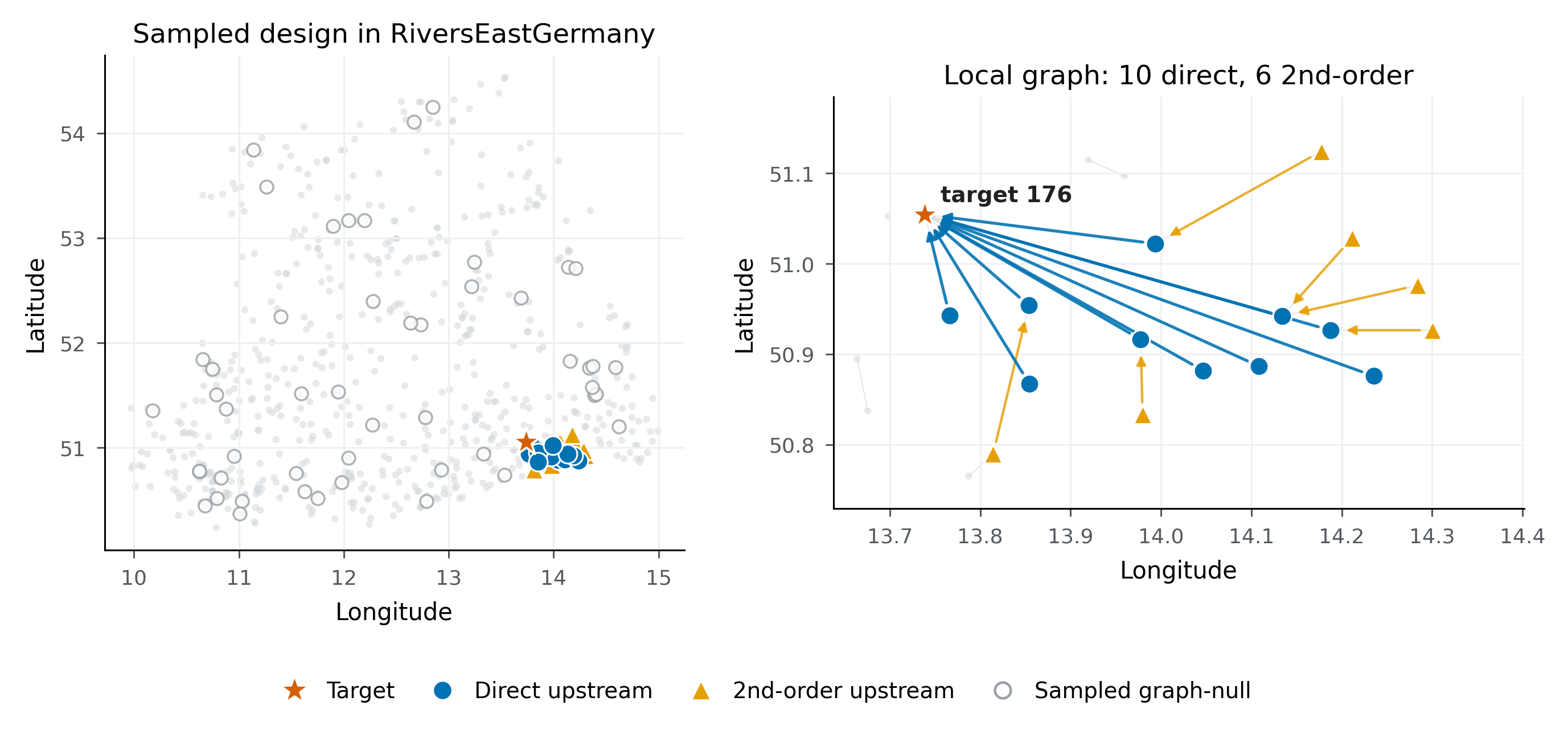"}
    \caption{Graph-proxy relevance construction example for Target ID 176}
    \label{causalriver_example_graph}
\end{figure}

\subsection{Real and Semi-synthetic Responses}

We evaluate both real and semi-synthetic responses, summarized in Table~\ref{tab:causalrivers_response_types}. The real response is \(y_t^{\mathrm{real}}=\Delta R_{t,j}\). Since the directed river graph does not fully determine the statistical null in our regression model, FDR and power for real \(y\) are reported as graph-proxy.

For semi-synthetic responses, we keep the real river covariates \(X\) fixed and define the true relevant set as \(S_1=\mathcal S_{\mathrm{direct}}\cup\mathcal S_{\mathrm{second}}\), where \(\mathcal S_{\mathrm{direct}}\) contains direct upstream groups and \(\mathcal S_{\mathrm{second}}\) contains second-order upstream groups. This preserves the real covariate dependence while allowing exact FDR and power evaluation.

For the linear semi-synthetic response, used in the linear benchmark, we generate
\[
y_t^{\mathrm{lin}}=\sum_{k\in S_1}X_{t,k}^{\top}\boldsymbol\beta_k+\varepsilon_t.
\]
Direct upstream groups have coefficient scale \(1.00\), second-order upstream groups have coefficient scale \(0.55\), and the lag profile is \((1.00,0.65,0.35)\) up to an independent random sign per relevant group. Gaussian noise is calibrated to signal-to-noise ratio \(=2.0\).

For the nonlinear semi-synthetic response, used in the neural benchmark, we first compute \(s_{t,k}=\boldsymbol w^\top X_{t,k}\), where \(\boldsymbol w\) is an \(\ell_2\)-normalized exponentially decaying lag-weight vector with larger weight on recent lags. We then generate
\[
\eta_t
=
\sum_{k\in \mathcal S_{\mathrm{direct}}}\xi_k \tanh(s_{t,k})
+
0.55\sum_{k\in \mathcal S_{\mathrm{second}}}\xi_k \sin(s_{t,k})
+
0.20\sum_{(k,k')\in \mathcal P_{\mathrm{direct}}}
\tanh(s_{t,k}s_{t,k'}/2),
\]
where \(\xi_k\in\{-1,1\}\) and \(\mathcal P_{\mathrm{direct}}\) contains pairs of direct upstream groups. We set \(y_t^{\mathrm{nonlin}}=\eta_t+\varepsilon_t\), with Gaussian noise calibrated to signal-to-noise ratio \(=2.0\). The interaction term uses only direct upstream groups, so the true support remains unambiguous.

\begin{table}[H]
\centering
\caption{Response types used in the CausalRivers application}
\label{tab:causalrivers_response_types}
\small
\begin{tabular}{llll}
\toprule
Outcome & Covariates & Support label & Main purpose \\
\midrule
Real \(y\) & Real river \(X\) & Graph-proxy & Graph-consistent real-data selection \\
Linear synthetic \(y\) & Real river \(X\) & Known support & Linear grouped-feature FDR evaluation \\
Nonlinear synthetic \(y\) & Real river \(X\) & Known support & Neural grouped-feature FDR evaluation \\
\bottomrule
\end{tabular}
\end{table}

\subsection{Methods and Evaluation Metrics}

For the CausalRivers experiments, we use target FDR level \(q=0.2\) to allow a less conservative real-data selection. For the linear benchmark, we compare GGM with Partial-\(F\) BH and cross-validated Group Lasso. Partial-\(F\) BH computes one group-level partial-\(F\) \(p\)-value for each station block and applies the Benjamini--Hochberg procedure to the resulting grouped-feature \(p\)-values. Group Lasso is treated as a feature-selection baseline rather than an FDR-control method.

For the neural benchmark, we train LSTM, GRU, and Transformer encoder models on the same station-level lag-block design. The training split is further divided into fitting and validation blocks. We use hidden dimension 32, one layer, dropout 0.1, Adam optimizer with learning rate \(10^{-3}\), weight decay \(10^{-4}\), batch size 128, maximum 150 epochs, and early stopping patience 20. The Transformer uses 4 attention heads. For PSGM selection, we use grouped Permutation-SHAP importance with 10 permutation draws and train the selection model for 120 epochs.

For each method, we report FDR or proxy FDR, power or proxy power, selected-group rates for direct upstream, second-order upstream, and null groups, and test MSE. The columns \(\%\) Direct, \(\%\) 2nd, and \(\%\) Null denote the fraction of groups selected within each category. We also report the full-model MSE and an oracle upstream-only MSE. For real \(y\), this oracle is a graph-oracle model using direct and second-order upstream groups; for semi-synthetic \(y\), it is the true-support oracle.

\begin{table}[H]
\centering
\caption{Hyperparameters for CausalRivers experiments}
\label{tab:causalrivers_hyperparameters}
\small
\begin{tabular}{ll}
\toprule
Component & Value \\
\midrule
Resampling frequency & 3 hours \\
Look-back length & \(L=3\) \\
Target FDR level & \(q=0.2\) \\
Number of target stations & 6 \\
Graph-null groups per target-seed case & 50 \\
Train-test split & \(4:1\), chronological \\
Neural hidden dimension & 32 \\
Neural layers & 1 \\
Dropout & 0.1 \\
Learning rate & \(10^{-3}\) \\
Weight decay & \(10^{-4}\) \\
Batch size & 128 \\
Maximum epochs & 150 \\
Early stopping patience & 20 \\
Transformer heads & 4 \\
PSGM permutation draws & 10 \\
PSGM selection epochs & 120 \\
Synthetic SNR & 2.0 \\
\bottomrule
\end{tabular}
\end{table}

\subsection{Linear CausalRivers Results}

Table~\ref{tab:causalrivers_linear} reports the linear CausalRivers results for both real \(y\) and the linear semi-synthetic response. For real \(y\), FDR and power are graph-proxy metrics based on the river graph labels. For semi-synthetic \(y\), FDR and power are computed using the constructed true support. The table also reports category-wise selection rates and test MSE for selected, full, and oracle upstream-only models.

\begin{table}[H]
\centering
\footnotesize
\setlength{\tabcolsep}{3.5pt}
\renewcommand{\arraystretch}{0.9}
\caption{Linear model results on the CausalRivers real data application}
\label{tab:causalrivers_linear}
\begin{tabular}{@{}cccccccccc@{}}
\toprule
Outcome & Method 
& FDR/Proxy
& Power/Proxy
& \% Direct 
& \% 2nd  
& \% Null
& Sel. MSE 
& Full MSE 
& Oracle MSE \\
\midrule
\multirow{3}{*}{\makecell{Real \(y\)}}
& GGM 
& \textcolor{blue}{0.508} & 0.194 & 0.312 & 0.137 & 0.116 & 1.437 & 1.485 & 1.342 \\
& Partial-F BH 
& 0.576 & \textcolor{blue}{0.342} & 0.431 & 0.293 & 0.148 & \textcolor{blue}{1.302} & 1.485 & 1.342 \\
& Group Lasso 
& \textbf{0.415} & \textbf{0.426} & 0.490 & 0.390 & 0.238 & \textbf{1.275} & 1.485 & 1.342 \\
\midrule
\multirow{3}{*}{\makecell{Semi-\\synthetic \(y\)}}
& GGM 
& \textcolor{blue}{0.219} & \textcolor{blue}{0.922} & 0.974 & 0.885 & 0.144 & 0.400 & 0.394 & 0.358 \\
& Partial-F BH 
& \textbf{0.127} & 0.885 & 0.956 & 0.825 & 0.050 & \textcolor{blue}{0.397} & 0.394 & 0.358 \\
& Group Lasso 
& 0.668 & \textbf{0.976} & 0.978 & 0.972 & 0.627 & \textbf{0.381} & 0.394 & 0.358 \\
\bottomrule
\end{tabular}
\end{table}

\clearpage

\newpage

\section*{NeurIPS Paper Checklist}

\begin{enumerate}

\item {\bf Claims}
    \item[] Question: Do the main claims made in the abstract and introduction accurately reflect the paper's contributions and scope?
    \item[] Answer: \answerYes{}
    \item[] Justification: Our contributions as stated in the abstract and introduction are backed by our theoretical justifications and experimental results in later parts.

\item {\bf Limitations}
    \item[] Question: Does the paper discuss the limitations of the work performed by the authors?
    \item[] Answer: \answerYes{}
    \item[] Justification: Please see details in the Conclusion and Limitations section (Section \ref{conclusion}).

\item {\bf Theory assumptions and proofs}
    \item[] Question: For each theoretical result, does the paper provide the full set of assumptions and a complete (and correct) proof?
    \item[] Answer: \answerYes{}
    \item[] Justification: Please see the methods section (\ref{linear_ggm} and \ref{neuralnetwork}) for the theoretical results and Appendix \ref{proofs} for corresponding proofs.

    \item {\bf Experimental result reproducibility}
    \item[] Question: Does the paper fully disclose all the information needed to reproduce the main experimental results of the paper to the extent that it affects the main claims and/or conclusions of the paper (regardless of whether the code and data are provided or not)?
    \item[] Answer: \answerYes{}
    \item[] Justification: We provide detailed algorithms in Appendix \ref{algorithms}. We also display all the necessary parameter and model settings to reproduce our numerical simulations and real data experiments in the corresponding sections.

\item {\bf Open access to data and code}
    \item[] Question: Does the paper provide open access to the data and code, with sufficient instructions to faithfully reproduce the main experimental results, as described in supplemental material?
    \item[] Answer: \answerYes{}
    \item[] Justification: Code is available at: \url{https://github.com/HJA321/Model\_Agnostic\_FDR\_Control}.

\item {\bf Experimental setting/details}
    \item[] Question: Does the paper specify all the training and test details (e.g., data splits, hyperparameters, how they were chosen, type of optimizer) necessary to understand the results?
    \item[] Answer: \answerYes{}
    \item[] Justification: We display all the necessary parameter and model settings to reproduce our numerical simulations and real data experiments in the sections \ref{simulation}, \ref{realdata} and Appendix \ref{causalriverdetails}.

\item {\bf Experiment statistical significance}
    \item[] Question: Does the paper report error bars suitably and correctly defined or other appropriate information about the statistical significance of the experiments?
    \item[] Answer: \answerYes{} 
    \item[] Justification: For all experiments, we report FDR and power as an estimation of the means computed from independent and repeated experiments.

\item {\bf Experiments compute resources}
    \item[] Question: For each experiment, does the paper provide sufficient information on the computer resources (type of compute workers, memory, time of execution) needed to reproduce the experiments?
    \item[] Answer: \answerYes{}
    \item[] Justification: All experiments were carried out on a Supermicro AS-4124GO-NART+ server equipped with two 64-core AMD EPYC 7763 CPUs, 2 TiB of RAM, and eight NVIDIA A100-SXM4-80GB GPUs.
    
\item {\bf Code of ethics}
    \item[] Question: Does the research conducted in the paper conform, in every respect, with the NeurIPS Code of Ethics \url{https://neurips.cc/public/EthicsGuidelines}?
    \item[] Answer: \answerYes{}
    \item[] Justification: We have clearly reviewed the Code of Ethics and made sure we conform to it.

\item {\bf Broader impacts}
    \item[] Question: Does the paper discuss both potential positive societal impacts and negative societal impacts of the work performed?
    \item[] Answer: \answerYes{}
    \item[] Justification: The potential impact and the importance of our work is discussed in the introduction section. As most of our work are theoretical, our method does not propose any specific ethical problem or negative societal impact apart from the common dangers in this area.
    
\item {\bf Safeguards}
    \item[] Question: Does the paper describe safeguards that have been put in place for responsible release of data or models that have a high risk for misuse (e.g., pre-trained language models, image generators, or scraped datasets)?
    \item[] Answer: \answerNA{}
    \item[] Justification: This paper does not pose risks related to this.

\item {\bf Licenses for existing assets}
    \item[] Question: Are the creators or original owners of assets (e.g., code, data, models), used in the paper, properly credited and are the license and terms of use explicitly mentioned and properly respected?
    \item[] Answer: \answerYes{}
    \item[] Justification: Existing works, codes and datasets are properly cited.

\item {\bf New assets}
    \item[] Question: Are new assets introduced in the paper well documented and is the documentation provided alongside the assets?
    \item[] Answer: \answerYes{}
    \item[] Justification: The main assets of this paper are theoretical results and codes for algorithms. The codes will be released upon acceptance.

\item {\bf Crowdsourcing and research with human subjects}
    \item[] Question: For crowdsourcing experiments and research with human subjects, does the paper include the full text of instructions given to participants and screenshots, if applicable, as well as details about compensation (if any)? 
    \item[] Answer: \answerNA{}
    \item[] Justification: 
    Our paper does not involve crowdsourcing nor research with human subjects.

\item {\bf Institutional review board (IRB) approvals or equivalent for research with human subjects}
    \item[] Question: Does the paper describe potential risks incurred by study participants, whether such risks were disclosed to the subjects, and whether Institutional Review Board (IRB) approvals (or an equivalent approval/review based on the requirements of your country or institution) were obtained?
    \item[] Answer: \answerNA{}
    \item[] Justification: Our work does not involve such risks.

\item {\bf Declaration of LLM usage}
    \item[] Question: Does the paper describe the usage of LLMs if it is an important, original, or non-standard component of the core methods in this research? Note that if the LLM is used only for writing, editing, or formatting purposes and does \emph{not} impact the core methodology, scientific rigor, or originality of the research, declaration is not required.
    \item[] Answer: \answerNA{}
    \item[] Justification: LLM is used only for writing, editing, or formatting purposes in our paper.

\end{enumerate}

\end{document}